\documentclass[11pt]{article}

\PassOptionsToPackage{hyphens}{url}

\usepackage[final]{acl}

\usepackage{times}
\usepackage{latexsym}
\usepackage[T1]{fontenc}
\usepackage[utf8]{inputenc}
\usepackage{microtype}
\usepackage{inconsolata}

\usepackage{xurl}
\usepackage{graphicx}
\usepackage{booktabs}
\usepackage{multirow}
\usepackage{float}

\usepackage{amsfonts}
\usepackage{amsmath}
\usepackage{amssymb}
\usepackage{amsthm}
\usepackage{nicefrac}

\usepackage{algorithm}
\usepackage{algpseudocode}
\usepackage{enumitem}

\newtheorem{proposition}{Proposition}

\title{Rollout-Level Advantage-Prioritized Experience Replay for GRPO}

\author{
  Gyeongtae Yoo$^{1}$ \quad Sanghyeok Park$^{2}$ \quad Soohyuk Jang$^{1}$ \quad Ik-hwan Kim$^{1}$ \\
  Sungroh Yoon$^{1,2,3}$\stepcounter{footnote}\thanks{Corresponding author.} \\
  $^{1}$Department of Electrical and Computer Engineering, Seoul National University \\
  $^{2}$Interdisciplinary Program in AI, Seoul National University \\
  $^{3}$AIIS, ASRI, INMC, and ISRC, Seoul National University \\
  \texttt{sryoon@snu.ac.kr}
}

\begin{document}
\maketitle

\begin{abstract}
Reinforcement learning from verifiable rewards with GRPO is a standard approach for post-training reasoning LLMs.
It remains sample inefficient.
Each rollout is used for a single gradient update and then discarded.
Naive replay is not well suited in this setting because LLM policies drift quickly per gradient step.
Stored rollouts therefore become stale and can destabilize training.
We propose a rollout-level replay buffer for GRPO that stores and samples individual rollouts rather than whole groups.
The buffer bounds staleness through age eviction.
Any rollout older than $\tau_{\max}$ training steps is removed.
The buffer also preserves on-policy data via fresh-anchored composition.
Each batch keeps its fresh on-policy rollouts and then concatenates replay rollouts drawn separately from the buffer.
We prioritize replay by per-rollout advantage magnitude and recycle individual rollouts whose advantages are large.
Across three Qwen3-Base scales on five math benchmarks, our method outperforms GRPO and naive replay baselines.
Gains are positive at every scale and reach $+1.66$~pp on the five-benchmark average at 4B.
Under an AES metric that jointly measures accuracy and token efficiency, our method is the only condition with a positive margin over GRPO at every scale.
\end{abstract}

\section{Introduction}
\label{sec:intro}

Reinforcement learning from verifiable rewards (RLVR) with Group Relative Policy Optimization (GRPO)~\citep{deepseekmath, deepseekr1} is a standard approach for post-training reasoning LLMs.
Producing a group of candidate rollouts per training prompt accounts for most of GRPO's training cost~\citep{prompt-replay, greso}.
Yet each rollout drives only a single gradient update before it is discarded.
This is especially wasteful for rollouts that carry a strong learning signal.
One example is a rare correct rollout in an otherwise-incorrect group.

Experience replay is one option, but it does not transfer directly to GRPO.
Policies drift quickly per gradient step, and stored rollouts become stale within a few updates~\citep{areal, m2po}.
Such staleness can destabilize training, so explicit staleness control is required in any GRPO replay design.

We propose a \textbf{rollout-level} replay buffer for GRPO, where individual rollouts are stored and sampled.
The buffer bounds staleness via \textbf{age eviction}.
Each rollout's age is the number of training steps elapsed since it was generated.
Any rollout older than $\tau_{\max}$ steps is removed.
The policy lag of any reused rollout is therefore bounded by $\tau_{\max}$ steps.
This step-level bound holds independently of buffer size or model scale, unlike a capacity-based rule.

The buffer-entry and eviction policy alone do not specify how the gradient batch is constructed.
A single unified buffer of fresh and stale rollouts can displace the fresh on-policy rollouts at a given step.
This is the standard DQN and PER pattern, and we refer to it as \textbf{pool composition}.
We instead adopt the \textbf{fresh-anchored composition} used in prior GRPO replay work~\citep{exgrpo, repo}.
The fresh rollouts are retained in full.
Replay rollouts are drawn separately and then concatenated on top.
Each gradient step then includes an on-policy anchor.
Age eviction and fresh-anchored composition together bound worst-case policy lag.

Within each group, the standard normalization yields a per-rollout group-relative advantage $A_i$.
Rollouts with the largest absolute advantage $|A_i|$ account for most of the learning signal~\citep{cppo, dppo-prune}.
These are the within-group minority rollouts in skewed groups.
We adopt $|A_i|$ as the replay priority signal and sample individual rollouts directly across the buffer.
This contrasts with prior GRPO replay schemes that operate at the granularity of the query.
Those schemes store and sample the $G$ rollouts of each prompt together as a single bound unit, a setting we refer to as \textbf{query-level} replay~\citep{exgrpo, repo}.

We train on the DeepScaleR-Preview prompt set~\citep{deepscaler} at three Qwen3-Base scales of 0.6B, 1.7B, and 4B.
Our method outperforms GRPO and two naive replay baselines on five math benchmarks.
Gains are positive at every scale and reach $+1.66$~pp on the five-benchmark average at 4B.
The contributions of our work are as follows.
\begin{itemize}[nosep,leftmargin=*]
\item We propose a GRPO replay buffer that combines age eviction, fresh-anchored composition, and rollout-level $|A_i|$ priority.
      Our method outperforms GRPO and naive replay baselines at every scale.
      Five-benchmark-average gains are $+0.61$~pp at 0.6B, $+0.60$~pp at 1.7B, and $+1.66$~pp at 4B.
      See Section~\ref{sec:method}, Section~\ref{sec:exp-main}, and Table~\ref{tab:cross-scale-main}.
  \item Per-axis ablations at 1.7B justify each component of the method.
        A moderate age cap of $\tau_{\max}=10$ performs best and trades off replay volume against staleness.
        Fresh-anchored composition is necessary at scale.
        Its pool-composition alternative mixes stale rollouts with fresh ones and falls below GRPO at 4B.
        Rollout-level $|A_i|$ priority further leads on PASS@$K$.
        See Section~\ref{sec:exp-design} and Tables~\ref{tab:age-clipping}, \ref{tab:pool-vs-concat-cross-scale}, and \ref{tab:priority-signal}.
  \item Our method also improves the accuracy efficiency score (AES) over GRPO at every scale.
        AES jointly scores accuracy and generation-token budget as a single efficiency value.
        A length-by-correctness decomposition further analyzes the gain.
        See Section~\ref{sec:dynamics} and Tables~\ref{tab:aes-base} and \ref{tab:length-quality}.
\end{itemize}

\begin{figure*}[!t]
    \centering
    \includegraphics[width=\textwidth]{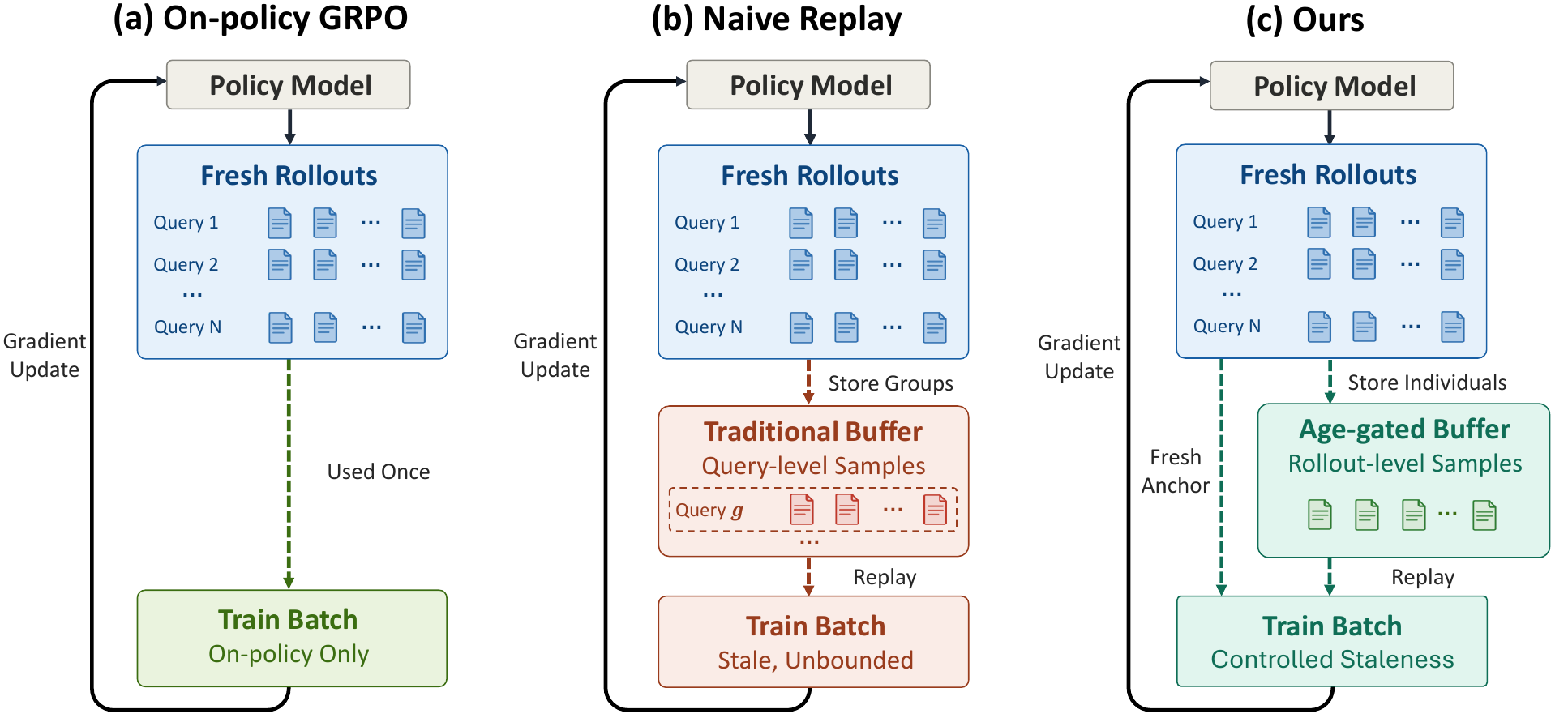}
    \caption{RL post-training paradigms for LLMs.
             Panel (a) is strictly on-policy GRPO that discards each rollout.
             Panel (b) is classical replay that stores whole groups under a capacity bound with unbounded staleness.
             Panel (c) is our method, which anchors every update with a fresh batch and replays individual rollouts from an age-gated buffer under $|A_i|$ priority.}
    \label{fig:method-overview}
\end{figure*}

\section{Related Work}
\label{sec:related}

\paragraph{Experience Replay with DQN and PER}
Experience replay originated with DQN~\citep{dqn}.
DQN trains a value network on transitions sampled from a FIFO buffer so that the network draws past experience rather than only the latest trajectory.
Prioritized Experience Replay~\citep{per} extends DQN with sampling proportional to a power-scaled priority.
The priority is typically the TD-error and focuses updates on transitions with the largest learning signal.
Both methods are built for value-based RL with millions of small refreshable transitions.
That regime does not hold for GRPO.
Our buffer adopts the DQN-style FIFO and the PER backbone, and adds GRPO-specific adaptations developed in Section~\ref{sec:method}.

\paragraph{Experience Replay for LLM RL}
Applying replay to GRPO violates two PER assumptions.
First, the group-relative advantage $A_i$ cannot be refreshed without regenerating the entire group.
Second, LLM policies drift quickly per gradient step~\citep{areal, m2po}, so even recent rollouts can be substantially stale.
The closest architectural neighbors are ExGRPO~\citep{exgrpo}, RePO~\citep{repo}, RLEP~\citep{rlep}, Buffer Matters~\citep{buffer-matters}, and \citet{fatemi2026prioritized}.
All five mix on-policy fresh rollouts with replayed rollouts in each GRPO update.
This on-policy anchor design is the same one that our fresh-anchored composition in Section~\ref{sec:method-buffer} adopts.
The five works differ from our method along the buffer-entry and priority axes, since all five replay at the query level.
ExGRPO prioritizes by correctness times entropy, RLEP by verified-correct retention, and Buffer Matters by re-evaluating historically difficult samples.
\citet{fatemi2026prioritized} prioritizes by empirical solve rate.
This is the maximum-variance regime equivalent to high $\sigma_g$.

Two other works apply a DQN-style replay framing for LLM RL.
\citet{arnal2026efficientrlreplay} maintains a FIFO trajectory buffer with uniform or correctness-weighted sampling.
That work uses an importance-ratio correction in place of an on-policy anchor.
FreshPER~\citep{freshper} applies a soft exponential age decay under reward-magnitude priority.
Both differ from our method along the priority and staleness-control axes.
\citet{arnal2026efficientrlreplay} also differs on composition.
Beyond these closest neighbors, the GRPO replay literature includes prompt-level reuse such as Prompt Replay~\citep{prompt-replay} and PSPO~\citep{pspo}.
It also includes filter-replay frameworks such as EFRame~\citep{efframe} and DOTS~\citep{dots}, and diversity-preserving replay such as DyJR~\citep{dyjr}.
These operate at the prompt level or interleave selection with replay, and are orthogonal to our rollout-level design.

\paragraph{Priority Signal Design}
Two angles in recent GRPO work inform priority signal design.
The first is absolute advantage as a selection signal.
CPPO~\citep{cppo} prunes completions with low $|A|$ since the high $|A|$ tail accounts for the bulk of the useful learning signal.
DPPO~\citep{dppo-prune} adds an importance-sampling correction under the same principle.
PODS~\citep{pods} downsamples each GRPO step's rollouts by reward variance.
It keeps the highest and lowest reward extremes within a group as a within-step analogue of high $|A_i|$ selection.
The second angle is within-group reward variance, widely used as a difficulty or informativeness proxy.
DAPO~\citep{dapo} drops zero-variance groups from training, a convention we also adopt in Section~\ref{sec:prelim}.
VCRL~\citep{vcrl} schedules training by group reward variance.
Online difficulty filtering~\citep{bae2026online, pcl} biases sampling toward prompts whose solve rate is bounded away from $0$ and $1$.
This is the maximum-variance regime under binary rewards and is equivalent to high $\sigma_g$.
The same regime is adopted as a replay-side priority by~\citet{fatemi2026prioritized}.
We use the query-level $\sigma_g$ as the baseline against our rollout-level $|A_i|$ priority.
Our work applies the first angle to replay sampling, since rollout-level $|A_i|$ priority targets the within-group minority rollouts that query-level $\sigma_g$ cannot resolve.
Appendix~\ref{app:extended-related-work} expands on the classical-replay and priority-signal threads and summarizes the position of our method in the broader design space.

\section{Preliminaries}
\label{sec:prelim}

\begin{algorithm*}[!t]
\caption{Rollout-level $|A_i|$ replay training step at $t$.
         Fresh rollouts always anchor the update, and the replay draw recycles past rollouts under an $|A_i|$-prioritized buffer with age eviction.}
\label{alg:ours}
\begin{algorithmic}[1]
\Require Current policy $\pi_t$,\; fresh-rollout count $B_{\text{fresh}}$,\; replay ratio $r$,\; max age $\tau_{\max}$

\Statex \textbf{Phase 1.} Fresh rollout generation
\State Sample $B_{\text{fresh}}$ fresh rollouts from $\pi_t$ and compute $\{A_i\}$ following Section~\ref{sec:prelim}
\State \textbf{Filter.} drop zero-variance groups, leaving $B_{\text{fresh}}'$ surviving rollouts \hfill\Comment{zero-variance filter}

\Statex \textbf{Phase 2.} $|A_i|$-prioritized replay sampling
\State Draw $r\,B_{\text{fresh}}'$ rollouts from the buffer via Equation~\ref{eq:per-sampling} \hfill\Comment{rollout-level}

\Statex \textbf{Phase 3.} Mixed-policy gradient update
\State $\pi_{t+1} \gets$ GRPO update on the fresh survivors concatenated with the replay draw \hfill\Comment{fresh-anchored}

\Statex \textbf{Phase 4.} Buffer update with age-gated eviction
\State Add the fresh survivors to the buffer with birth step $t$
\State Evict every rollout with birth step $\,<\,t-\tau_{\max}$ \hfill\Comment{age eviction}
\end{algorithmic}
\end{algorithm*}

We study GRPO~\citep{deepseekmath} under binary rewards.

\paragraph{GRPO}
At training step $t$ we pick a prompt $q$ and ask the current policy $\pi_t$ to generate a group of $G$ candidate answers $\{o_1, \ldots, o_G\}$.
Each $o_i$ is a rollout.
An automatic verifier assigns each rollout a binary reward $r_i \in \{-1, +1\}$.
The group-relative advantage of rollout $i$ is
\begin{equation}
\label{eq:grpo-adv}
  A_i \;=\; \frac{r_i - \mu_g}{\sigma_g + \epsilon_\sigma},
\end{equation}
where $\mu_g$ and $\sigma_g$ are the within-group reward mean and standard deviation, and $\epsilon_\sigma$ is a small numerical-stability constant.
Groups with identical rewards yield $A_i = 0$ for all members and provide no learning signal.
We therefore drop such groups, so the replay buffer holds only nonzero-advantage rollouts.
This keeps buffer management simple.
GRPO updates $\pi_\theta$ via a PPO-style clipped surrogate with importance ratio $\rho_i = \frac{\pi_\theta(o_i\mid q)}{\pi_t(o_i\mid q)}$,
\begin{multline}
\label{eq:grpo-obj}
  \mathcal{J}_{\text{GRPO}}(\theta)
  \;=\;
  \mathbb{E}_{q, \{o_i\}}
  \Bigg[
    \tfrac{1}{G}
    \sum_{i=1}^{G}
    \min\!\big(
      \rho_i A_i, \\
      \mathrm{clip}(\rho_i,\,
                    1 - \epsilon_{\text{low}},\,
                    1 + \epsilon_{\text{high}})\, A_i
    \big)
  \Bigg].
\end{multline}
Here $\pi_t$ denotes the behavior policy at the rollout's generation step.
$\pi_\theta$ denotes the current target policy at the update step.
$\rho_i$ is applied per token following the verl and DAPO convention.
For replay, $\pi_t(o_i\!\mid\! q)$ is cached with each rollout at generation time, so $\rho_i$ reflects drift from the rollout's birth step to the current training step.
We adopt the Clip-Higher scheme of DAPO~\citep{dapo} and set $\epsilon_{\text{low}} = 0.2$ and $\epsilon_{\text{high}} = 0.28$.

\paragraph{Prioritized Replay Sampling}
We use the PER scheme~\citep{per} with a fixed-capacity FIFO buffer that evicts the oldest entry when full.
Section~\ref{sec:method-buffer} extends this with age eviction.
A unit $i$ with priority $p_i \geq 0$ is drawn with probability
\begin{equation}
\label{eq:per-sampling}
  P(i) \;=\; \frac{p_i^{\alpha}}{\sum_j p_j^{\alpha}},
\end{equation}
where the sum runs over all units currently in the buffer.
The exponent $\alpha \in [0,1]$ interpolates between uniform sampling at $\alpha = 0$ and fully proportional sampling at $\alpha = 1$.
We set the priority to the advantage magnitude $p_i = |A_i|$.
Section~\ref{sec:method-priority} motivates this choice.

\section{Adapting Experience Replay to GRPO}
\label{sec:method}

\begin{table*}[t]
\centering\scriptsize
\setlength{\tabcolsep}{3pt}
\renewcommand{\arraystretch}{0.9}
\resizebox{\ifdim\width>\textwidth \textwidth\else \width\fi}{!}{%
\begin{tabular}{l l cc cc cc cc cc cc}
\toprule
\multirow{2}{*}[-0.5ex]{Model} & \multirow{2}{*}[-0.5ex]{Method}
 & \multicolumn{2}{c}{MATH-500}
 & \multicolumn{2}{c}{AIME25}
 & \multicolumn{2}{c}{AIME26}
 & \multicolumn{2}{c}{HMMT-F25}
 & \multicolumn{2}{c}{HMMT-F26}
 & \multicolumn{2}{c}{avg} \\
\cmidrule(lr){3-4}\cmidrule(lr){5-6}\cmidrule(lr){7-8}\cmidrule(lr){9-10}\cmidrule(lr){11-12}\cmidrule(lr){13-14}
 & & avg@8 & pass@8
 & avg@16 & pass@16
 & avg@16 & pass@16
 & avg@16 & pass@16
 & avg@16 & pass@16
 & avg & pass \\
\midrule
\multirow{4}{*}{Qwen3-0.6B}
  & GRPO & 52.71 & 74.27 & 1.67 & 14.44 & 0.28 & 3.33 & 0.35 & 4.44 & 0.88 & \textbf{10.10} & 11.18 $\pm$ 0.19 & 21.32 $\pm$ 0.40 \\
  & uniform & 52.89 & \textbf{76.87} & 1.94 & 12.22 & 0.90 & \textbf{7.78} & 0.21 & 3.33 & 0.57 & 6.06 & 11.30 $\pm$ 0.17 & 21.25 $\pm$ 1.24 \\
  & $\sigma_g$ & 52.32 & \textbf{76.87} & \textbf{2.50} & 14.44 & 0.90 & 6.67 & 0.28 & 4.44 & 1.14 & 9.09 & 11.43 $\pm$ 0.38 & 22.30 $\pm$ 0.92 \\
  & Ours & \textbf{53.53} & \textbf{76.87} & 2.08 & \textbf{15.56} & \textbf{1.18} & 5.56 & \textbf{0.62} & \textbf{5.56} & \textbf{1.52} & 9.09 & \textbf{11.79 $\pm$ 0.25} & \textbf{22.52 $\pm$ 0.78} \\
\midrule
\multirow{4}{*}{Qwen3-1.7B}
  & GRPO & 67.58 & 84.93 & 5.56 & 26.67 & 2.85 & 15.56 & 0.21 & 3.33 & \textbf{3.09} & 9.09 & 15.86 $\pm$ 0.10 & 27.92 $\pm$ 0.61 \\
  & uniform & 67.50 & 84.93 & 5.01 & 24.84 & 3.86 & 13.76 & 0.38 & 4.32 & 2.11 & 9.09 & 15.77 $\pm$ 0.41 & 27.39 $\pm$ 0.63 \\
  & $\sigma_g$ & 67.43 & 85.40 & 4.03 & 16.67 & 4.24 & \textbf{20.00} & 0.21 & 3.33 & \textbf{3.09} & \textbf{11.11} & 15.80 $\pm$ 0.40 & 27.30 $\pm$ 2.09 \\
  & Ours & \textbf{68.87} & \textbf{86.07} & \textbf{5.83} & \textbf{30.00} & \textbf{4.44} & 17.78 & \textbf{0.69} & \textbf{6.67} & 2.46 & \textbf{11.11} & \textbf{16.46 $\pm$ 0.18} & \textbf{30.32 $\pm$ 1.98} \\
\midrule
\multirow{4}{*}{Qwen3-4B}
  & GRPO & 83.31 & 92.80 & 17.57 & 38.89 & 17.29 & 32.22 & 8.89 & 20.00 & 11.81 & 26.26 & 27.77 $\pm$ 0.98 & 42.03 $\pm$ 1.28 \\
  & uniform & 82.86 & 92.93 & 17.19 & 38.19 & 15.13 & 32.95 & 8.58 & 27.38 & 10.93 & 24.96 & 26.94 $\pm$ 1.70 & 43.28 $\pm$ 0.68 \\
  & $\sigma_g$ & 82.25 & 92.67 & 16.60 & 42.22 & 14.65 & 33.33 & 6.39 & 20.00 & 9.85 & 24.24 & 25.95 $\pm$ 0.86 & 42.49 $\pm$ 3.12 \\
  & Ours & \textbf{84.52} & \textbf{94.13} & \textbf{19.38} & \textbf{43.33} & \textbf{19.44} & \textbf{45.56} & \textbf{10.14} & \textbf{31.11} & \textbf{13.70} & \textbf{30.30} & \textbf{29.43 $\pm$ 2.25} & \textbf{48.89 $\pm$ 3.71} \\
\bottomrule
\end{tabular}%
}
\caption{Cross-scale main results comparing GRPO, two naive replay baselines, and our fresh-anchored rollout-level $|A_i|$ method.
         The two naive baselines are uniform and $\sigma_g$ under pool composition.
         All replay variants share $r=0.5$, $\tau_{\max}=10$, and $\alpha=0.5$.
         MATH-500 uses $k=8$, and AIME and HMMT use $k=16$.
         The avg columns are the unweighted five-benchmark mean, reported as mean $\pm$ standard deviation across seeds.
         \textbf{Bold} marks the best per column within each scale.}
\label{tab:cross-scale-main}
\end{table*}

Applying experience replay to GRPO raises two issues.
First, replayed rollouts are stale relative to the current policy.
Second, the group structure of GRPO aggregates per-rollout learning signals into a single query-level value.
We address the first issue with age eviction and fresh-anchored composition, which together bound policy lag.
Section~\ref{sec:method-buffer} describes this design.
We address the second issue with rollout-level buffer entries and advantage-magnitude priority.
This selects the high-signal rollouts that query-level replay does not isolate, as described in Section~\ref{sec:method-priority}.
Figure~\ref{fig:method-overview} contrasts our method with standard on-policy GRPO and classical replay.
The full procedure is given in Algorithm~\ref{alg:ours}.

Replay makes the update only approximately on-policy, and the replayed portion of each gradient step is a biased estimator of the on-policy GRPO gradient.
We do not correct that bias.
The controls below bound that departure rather than remove it, and the cached $\pi_t(o_i \mid q)$ of a replayed rollout enters only through the clipped ratio, which limits its influence without reweighting it to the current policy.
Our update therefore does not recover the on-policy objective.

\subsection{Age Eviction and Fresh-Anchored Composition}
\label{sec:method-buffer}

\paragraph{Age Eviction}
To bound per-rollout staleness, each stored rollout carries a birth step $t_b$ that records the training step at which it was generated.
Its age at the current step $t$ is $t - t_b$.
We impose age eviction and remove any rollout whose age exceeds $\tau_{\max}$, so no update reuses a rollout beyond a fixed age.
Bounding by age rather than buffer capacity keeps the limit invariant to model scale and data difficulty.
Zero-variance filtering removes a scale-dependent fraction of generated groups, so the effective batch $B_{\text{fresh}}'$ varies with model scale.
Appendix~\ref{app:effective-batch} reports the per-scale values.
Any fixed capacity therefore induces different steady-state ages per setting.
In contrast, $\tau_{\max}$ matches the per-step drift of LLM post-training~\citep{areal, m2po} without per-configuration tuning.
A FIFO capacity can only approximate this because $B_{\text{fresh}}'$ shifts both across settings and during a single run as the solve-rate distribution of the policy changes.\footnote{%
A capacity cap of $30{,}000$ is retained as a backstop.
At the default $\tau_{\max} = 10$, age eviction is the binding constraint at all scales.
The steady-state buffer size $\tau_{\max} \cdot B_{\text{fresh}}'$ is about $10{,}000$ even at 4B.%
}
A brief warmup with replay disabled at the start of training allows the buffer to populate and the model to acquire basic instruction-following before replay activates.
The warmup length is given in Section~\ref{sec:exp-setup}.

\paragraph{Fresh-Anchored Composition}
Age eviction caps the staleness of each individual rollout, but it does not specify how fresh and replay rollouts combine into the gradient batch.
At each training step we have fresh on-policy rollouts sampled from the current policy $\pi_t$ and the age-bounded buffer of past rollouts.
We adopt fresh-anchored composition, following recent GRPO replay schemes that pair fresh on-policy rollouts with separately drawn replay rollouts.
The contrasting mode merges fresh and stale rollouts into a single pool, which is the standard DQN and PER pattern.
We call this pool composition and include it as a naive baseline in Section~\ref{sec:exp-main}, then ablate it against fresh-anchored composition in Section~\ref{sec:exp-staleness}.
Our formulation parameterizes the mix by an explicit replay ratio $r$.
The fresh rollouts are retained in full.
A replay portion of size $r \cdot B_{\text{fresh}}'$ is drawn separately from the buffer and concatenated on top, where $B_{\text{fresh}}'$ is the post-filter fresh rollout count.
Each gradient step then retains a fresh on-policy anchor that bounds policy lag independently of the replay portion.
The pool-composition baseline uses the same total batch size $(1+r) \cdot B_{\text{fresh}}'$ but draws the entire batch from a unified pool of fresh and stale rollouts.
Any difference between the two modes in Section~\ref{sec:exp-staleness} is therefore attributable to composition mode rather than to replay budget.

\subsection{Rollout-Level $|A_i|$ Priority}
\label{sec:method-priority}

\paragraph{$|A_i|$ Priority}
Even with age eviction bounding policy lag (Section~\ref{sec:method-buffer}), each replayed rollout is by construction at least slightly off-policy.
The priority signal must therefore focus the replay budget on rollouts whose learning signal still justifies the off-policy cost.
We expect a useful priority to favor rollouts with (i) a large birth-step gradient coefficient and (ii) low representation in fresh batches.
The per-rollout advantage magnitude $|A_i|$ satisfies both.
It is the per-rollout coefficient in the GRPO objective (Equation~\ref{eq:grpo-obj}), the same role TD-error plays in value-based PER.
High-$|A_i|$ rollouts are also the within-group minority of skewed groups (Appendix~\ref{app:binary-advantage-sigma}) that fresh batches see sparsely, because such groups arise from prompts the policy rarely solves cleanly.
In contrast, $\sigma_g$ would emphasize balanced groups that fresh batches already cover densely.
Recent GRPO work uses $|A|$ similarly~\citep{cppo, dppo-prune, pods}.
Because sampling from the buffer follows $P(i) \propto p_i^{\alpha}$ (Equation~\ref{eq:per-sampling}), replay draws high-$|A_i|$ rollouts more often than the on-policy distribution would, and that reweighting is the purpose of the buffer rather than a side effect.
We consequently do not undo it with a PER importance-sampling weight, which at full strength targets the same expected update as uniform sampling from the buffer, the $\alpha = 0$ setting that Table~\ref{tab:priority-signal} reports as rollout-level uniform and places below $|A_i|$ priority on both metrics.
We set the strength of the reweighting with $\alpha$ instead and temper it at $\alpha = 0.5$, because full proportional sampling at $\alpha = 1.0$ overshoots and costs pass accuracy (Appendix~\ref{app:ratio-alpha-full}).

\paragraph{Rollout-Level Buffer Entries}
$|A_i|$ priority requires a buffer-entry decision.
Conventional GRPO replay schemes~\citep{exgrpo, freshper} bind the $G$ rollouts of a query into one buffer entry, so all $G$ rollouts share a single priority.
This forces the priority signal to be query-level, such as uniform or a group-aggregate like $\sigma_g$.
A query-level signal cannot resolve the within-group minority that holds the $|A_i|$ signal.
Combined with pool composition (Section~\ref{sec:method-buffer}), this configuration is the naive baseline in Table~\ref{tab:cross-scale-main}.
We therefore deviate from this convention and store individual rollouts.
Each rollout carries its own birth-step $|A_i|$ and is sampled directly by Equation~\ref{eq:per-sampling}.
Section~\ref{sec:exp-priority-signal} confirms that the gain comes from $|A_i|$'s within-group variation, not from the buffer-entry shift.
That coefficient stays at its birth-step value together with the group statistics $(\mu_g, \sigma_g)$, so it is estimated under an earlier solve-rate distribution while the gradient it multiplies is taken under the current policy, and rollout-level draws break the within-group centering that keeps the fresh portion mean-zero.
The same $\tau_{\max}$ that bounds policy lag also bounds how stale the coefficient can be, since a replayed advantage is never more than $\tau_{\max}$ steps out of date.

\section{Experiments}
\label{sec:exp}

\subsection{Experimental Setup}
\label{sec:exp-setup}

We train three scales of the Qwen3-Base family at 0.6B, 1.7B, and 4B.
Training uses the DeepScaleR-Preview~\citep{deepscaler} prompt set across all conditions.
The four conditions are GRPO, the two naive replay baselines, and our method.
We hold all training hyperparameters constant except for the replay-buffer design.
The cross-scale main comparison of Section~\ref{sec:exp-main} and the analyses of Section~\ref{sec:dynamics} are run with three training seeds, and the ablations of Section~\ref{sec:exp-design} use a single seed.
Training hyperparameters and the seed policy are listed in Appendix~\ref{app:training-setup}.

\paragraph{Replay Buffer}
Our method instantiates Algorithm~\ref{alg:ours} with fresh-anchored composition from Section~\ref{sec:method-buffer} and rollout-level $|A_i|$ priority from Section~\ref{sec:method-priority}.
The priority is $p_i = |A_i| + \epsilon$, and we sample per Equation~\ref{eq:per-sampling} with $\alpha = 0.5$.
Age eviction uses $\tau_{\max} = 10$ and the replay ratio is $r = 0.5$.
We use a $20$-step warmup during which replay is disabled.
The naive replay baselines are uniform and $\sigma_g$.
They share the same $\tau_{\max}$, $r$, $\alpha$, and warmup as our method.
They differ from our method in three ways.
They use pool composition, they store at the query level, and their priority signal is uniform or $\sigma_g$ rather than $|A_i|$.
A replayed rollout's advantage $A_i$ and its group statistics $(\mu_g, \sigma_g)$ are frozen at the birth-step values, since refreshing them would require regenerating the entire group.
We ablate $\tau_{\max}$ in Section~\ref{sec:exp-staleness} and jointly ablate $r$ and $\alpha$ in Appendix~\ref{app:ratio-alpha-full}.

\paragraph{Benchmarks and Decoding}
We evaluate on five math-reasoning benchmarks, namely MATH-500, AIME25, AIME26, HMMT-F25, and HMMT-F26.
Appendix~\ref{app:training-setup} (Benchmarks and Decoding paragraph) gives the per-benchmark sizes, decoding settings, and per-benchmark $k$ values.
That appendix also describes the five-benchmark aggregate used in the ablation tables from Table~\ref{tab:age-clipping} through Table~\ref{tab:priority-signal} and the full seed policy.

\begin{table}[t]
\centering\footnotesize
\setlength{\tabcolsep}{4pt}
\resizebox{0.85\columnwidth}{!}{%
\begin{tabular}{l cc cc}
\toprule
& \multicolumn{2}{c}{Accuracy} & \multicolumn{2}{c}{Clipping (\%)} \\
\cmidrule(lr){2-3}\cmidrule(lr){4-5}
Method & avg & pass & fresh & replay \\
\midrule
$\tau_{\max} = 1$       & 16.49 & 29.59 & 0.055 & 0.266 \\
\textbf{$\tau_{\max} = 10$} (Ours) & \textbf{16.58} & \textbf{32.35} & 0.060 & 0.403 \\
$\tau_{\max} = 30$      & 16.01 & 27.34 & 0.058 & 0.654 \\
$\tau_{\max} = 60$      & 15.28 & 26.66 & 0.058 & 0.875 \\
\bottomrule
\end{tabular}%
}
\caption{Age-eviction horizon $\tau_{\max}$ ablation at Qwen3-1.7B-Base.
         The clipping columns report the fraction of importance ratios clipped during the PPO update, decomposed into fresh and replay batch portions.
         \textbf{Bold} marks the best.}
\label{tab:age-clipping}
\end{table}

\subsection{Cross-Scale Main Comparison}
\label{sec:exp-main}

Table~\ref{tab:cross-scale-main} reports the main result.
The table separates the benefit of our method from the no-replay GRPO reference and from the naive replay baselines.
The naive baselines apply uniform and $\sigma_g$ priorities under pool composition with query-level buffer entries.
The naive baselines have their batch size matched to fresh-anchored ours, so any gap to our method is attributable to composition and rollout-level $|A_i|$ priority rather than to replay budget.

Our method improves over both baseline families at every scale.
The gain over GRPO on the five-benchmark average is $+0.61$~pp at 0.6B, $+0.60$~pp at 1.7B, and $+1.66$~pp at 4B.
The gain also extends to PASS@$K$.
Our method exceeds GRPO by $+6.86$~pp on pass, and the best naive baseline by $+2.49$~pp on avg and $+5.61$~pp on pass, as shown in the rightmost block of Table~\ref{tab:cross-scale-main}.

Naive replay gives no consistent gain over GRPO under either uniform or $\sigma_g$ priority, and at 4B it falls below the no-replay baseline by $0.83$ and $1.82$~pp on the five-benchmark average.
Reusing rollouts is therefore not beneficial on its own, and without a priority that selects which rollouts to reuse it can cost accuracy.
The design-space ablations in Section~\ref{sec:exp-design} examine which axes of our design separate it from these baselines.

\subsection{Ablations}
\label{sec:exp-design}

We validate each design choice of Section~\ref{sec:method} at 1.7B, with the composition ablation extended to 0.6B and 4B.
We sweep alternative variants along two axes.
The first axis is staleness handling and includes age eviction ($\tau_{\max}$) and fresh-anchored composition.
Section~\ref{sec:exp-staleness} covers this axis.
The second axis is the rollout-level priority signal, comparing $|A_i|$ against uniform and $\sigma_g$.
Section~\ref{sec:exp-priority-signal} covers this axis.
Each table reports five-benchmark aggregate avg and pass under the $k$ convention of Table~\ref{tab:cross-scale-main}.

\subsubsection{Age Eviction and Fresh Anchoring}
\label{sec:exp-staleness}
\begin{table}[t]
\centering\scriptsize
\setlength{\tabcolsep}{3pt}
\renewcommand{\arraystretch}{1.0}
\resizebox{\ifdim\width>\columnwidth \columnwidth\else \width\fi}{!}{%
\begin{tabular}{l cc cc cc}
\toprule
\multirow{2}{*}[-0.5ex]{Method, Composition}
 & \multicolumn{2}{c}{Qwen3-0.6B}
 & \multicolumn{2}{c}{Qwen3-1.7B}
 & \multicolumn{2}{c}{Qwen3-4B} \\
\cmidrule(lr){2-3}\cmidrule(lr){4-5}\cmidrule(lr){6-7}
 & avg & pass & avg & pass & avg & pass \\
\midrule
GRPO (no replay)                              & 11.20          & 21.32          & 15.75          & 27.64          & 26.71          & 40.92          \\
\midrule
Pool composition                              & 11.29          & 21.02          & 16.00          & 27.08          & 23.54          & 39.08          \\
Fresh-anchored (Ours)                & \textbf{11.58} & \textbf{23.24} & \textbf{16.58} & \textbf{32.35} & \textbf{31.06} & \textbf{50.98} \\
\bottomrule
\end{tabular}%
}
\caption{Cross-scale composition-mode comparison between fresh-anchored composition for ours and pool composition.
         Both are under rollout-level $|A_i|$ priority.
         avg and pass are five-benchmark aggregates under Table~\ref{tab:cross-scale-main}'s $k$ convention.
         \textbf{Bold} marks the better composition at each scale.}
\label{tab:pool-vs-concat-cross-scale}
\end{table}

Table~\ref{tab:age-clipping} sweeps the age-eviction horizon $\tau_{\max} \in \{1, 10, 30, 60\}$ at 1.7B with other axes at default.
Allowing samples to age past $\tau_{\max} = 10$ reduces both metrics.
As $\tau_{\max}$ grows from $10$ to $60$, avg drops from $16.58$ to $15.28$ and pass drops from $32.35$ to $26.66$.
Aging past the cap also induces asymmetric importance-ratio clipping.
Replay-side clipping grows monotonically with the cap and reaches $0.88\%$ at $\tau_{\max} = 60$.
Fresh-side clipping is approximately constant.
The asymmetric clipping alongside the performance drop indicates that an unmanaged buffer accumulates policy lag rather than stochastic noise.\footnote{%
All clipping rates stay below $1\%$.
Over $99\%$ of importance ratios therefore fall within the DAPO clip-higher window from $1 - 0.2$ to $1 + 0.28$.%
}
The per-step clipping trajectory is in Figure~\ref{fig:clipping-curves} of Appendix~\ref{app:clipping-curves}.
At the opposite end, $\tau_{\max} = 1$ matches our default on avg with $16.49$ against $16.58$, but is $-2.76$~pp lower on pass.
At $\tau_{\max} = 1$ the buffer holds at most one step's worth of rollouts at any time, so the replay portion has little to recycle and contributes minimal additional signal.
We adopt $\tau_{\max} = 10$ as the smallest age at which the buffer accumulates enough rollouts to make the priority-driven replay meaningful, while keeping policy lag bounded.

Age eviction alone does not fully bound policy lag at scale.
The composition mode also affects performance across the three tested scales.
Table~\ref{tab:pool-vs-concat-cross-scale} fixes $\tau_{\max} = 10$ and replaces fresh-anchored composition with pool composition.
The effect is scale-dependent.
At 0.6B the avg gap between the two compositions is negligible at $0.29$~pp and the pass gap is modest at $+2.22$~pp.
Both gaps widen from 0.6B to 4B.
At 1.7B the pass gain over GRPO is removed under pool composition, with $27.08$ against the GRPO value of $27.64$.
At 4B both avg and pass fall below GRPO outright, with avg $23.54$ against the GRPO value of $26.71$.

\subsubsection{Rollout-Level $|A_i|$ Against Query-Level Baselines}
\label{sec:exp-priority-signal}
\begin{table}[t]
\centering\footnotesize
\setlength{\tabcolsep}{12pt}
\resizebox{0.9\columnwidth}{!}{%
\begin{tabular}{l l cc}
\toprule
Method & Level & avg & pass \\
\midrule
GRPO (no replay) & n/a     & 15.75 & 27.64 \\
\midrule
\multirow{2}{*}{uniform}    & query   & 15.94 & 27.48 \\
                            & rollout & 16.19 & 28.96 \\
\midrule
\multirow{2}{*}{$\sigma_g$} & query   & 15.43 & 29.04 \\
                            & rollout & 16.04 & 29.16 \\
\midrule
Ours with $|A_i|$     & rollout & \textbf{16.58} & \textbf{32.35} \\
\bottomrule
\end{tabular}%
}
\caption{Priority signal and level ablation at Qwen3-1.7B-Base.
         All rows share fresh-anchored composition, $r=0.5$, and $\tau_{\max}=10$, so the rollout-level uniform row is an informative-sampling control that matches our method's replay volume and differs only in the priority signal.
         The $|A_i|$ signal is rollout-only by design.
         avg and pass are five-benchmark aggregates under Table~\ref{tab:cross-scale-main}'s $k$ convention.
         \textbf{Bold} marks the best per column.}
\label{tab:priority-signal}
\end{table}

Table~\ref{tab:priority-signal} ablates the priority signal at 1.7B across two axes.
The signal axis covers uniform, $\sigma_g$, and $|A_i|$.
The level axis covers query and rollout.
Uniform and $\sigma_g$ are normally query-level signals.
They can also be applied at the rollout level, which gives a small lift over their query-level counterparts.
The larger effect is the priority signal itself.
Rollout-level $|A_i|$ leads on both avg and pass, with the widest gap on pass at $+3.19$~pp over the next-best $\sigma_g$.
This indicates that rollout-level buffer entries alone are only a precondition.
The gain comes from a signal that varies within a group.
$|A_i|$ provides such variation by construction, while $\sigma_g$ is group-constant and does not.
Rollout-level $|A_i|$ and rollout-level uniform draw the same number of replay rollouts from a buffer that admits only nonzero-advantage rollouts, so the $+0.39$~pp avg and $+3.39$~pp pass gap between them reflects which rollouts are selected rather than how many informative ones are added.
Replay-free informative sampling such as dynamic sampling and oversample-and-filter raises the share of nonzero-advantage groups by generating more rollouts, so it pays on generation while replay pays on the learner.
Holding the number of informative rollouts fixed, as this control does, therefore separates prioritization from that filtering within a fixed generation budget.
We adopt rollout-level $|A_i|$.

\section{Response-Length and Efficiency Analysis}
\label{sec:dynamics}

The effect of replay in this setting is not limited to accuracy.
We additionally evaluate each post-trained model with AES~\citep{o1-pruner}.
AES was originally proposed to score reasoning-model efficiency as the joint trade-off between accuracy and generation length.

\subsection{Accuracy Efficiency Score (AES)}
\label{sec:dyn-aes}
\begin{table}[t]
\centering\footnotesize
\setlength{\tabcolsep}{5pt}
\resizebox{\ifdim\width>\columnwidth \columnwidth\else \width\fi}{!}{%
\begin{tabular}{l ccc}
\toprule
\multirow{2}{*}{Method}
 & \multicolumn{3}{c}{AES per Qwen3-Base scale} \\
\cmidrule(lr){2-4}
 & 0.6B & 1.7B & 4B \\
\midrule
\multicolumn{4}{l}{\emph{Against untrained base model}} \\
GRPO          & $-0.146$ & $+0.474$ & $-0.067$ \\
\textbf{Ours} & $\mathbf{+0.058}$ & $\mathbf{+0.718}$ & $\mathbf{+0.585}$ \\
\midrule
\multicolumn{4}{l}{\emph{Against no-replay GRPO baseline}} \\
uniform       & $+0.016$ & $+0.008$ & $-0.001$ \\
$\sigma_g$    & $-0.245$ & $+0.046$ & $-0.075$ \\
\textbf{Ours} & $\mathbf{+0.177}$ & $\mathbf{+0.221}$ & $\mathbf{+0.292}$ \\
\bottomrule
\end{tabular}%
}
\caption{Token-based AES at three Qwen3-Base scales.
         The top block is referenced to the untrained Base init and the bottom block to the no-replay GRPO baseline.
         Reference values are listed in Appendix~\ref{app:aes-references}.
         Naive baselines use pool composition with query-level priority, and \textbf{bold} marks the best AES per scale.}
\label{tab:aes-base}
\end{table}

\begin{table}[t]
\centering\footnotesize
\setlength{\tabcolsep}{6pt}
\resizebox{\ifdim\width>\columnwidth \columnwidth\else \width\fi}{!}{%
\begin{tabular}{l ccc}
\toprule
\multirow{2}{*}{Scale} & \multicolumn{2}{c}{$\mu_{\checkmark} / \mu_{\times}$} & \multirow{2}{*}{$\Delta$ (\%)} \\
\cmidrule(lr){2-3}
 & GRPO & Ours & \\
\midrule
Qwen3-0.6B & $0.139$ & $0.185$ & $\mathbf{+33.6}$ \\
Qwen3-1.7B & $0.224$ & $0.277$ & $\mathbf{+23.4}$ \\
Qwen3-4B & $0.119$ & $0.147$ & $\mathbf{+25.2}$ \\
\bottomrule
\end{tabular}%
}
\caption{Ratio of mean response length on correct rollouts to mean response length on incorrect rollouts, pooled across the five-benchmark set.
         Lengths are pooled over rollouts, and $\Delta$ is the relative change of that ratio from GRPO to ours.
         A positive $\Delta$ means correct rollouts grew longer relative to incorrect ones.}
\label{tab:length-quality}
\end{table}

\begin{figure*}[t]
  \centering
  \includegraphics[width=\textwidth]{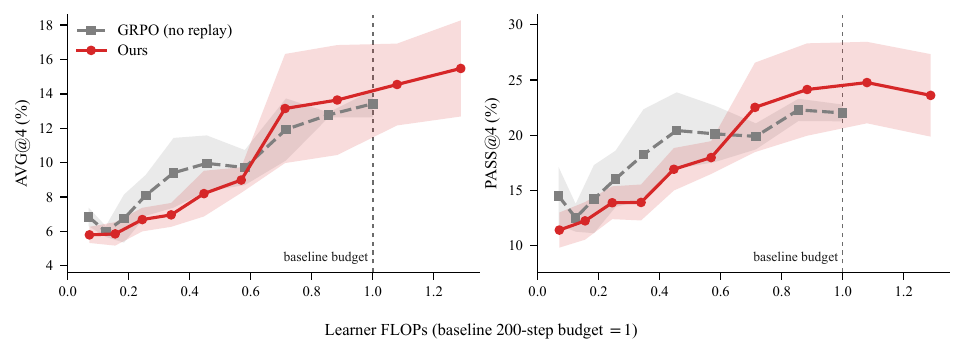}
  \caption{Accuracy against learner FLOPs at Qwen3-4B-Base over three training seeds, normalized so each seed's baseline $200$-step budget is $1$.
           The left panel is AVG@$K$ and the right is PASS@$K$.
           Points are evaluated every $20$ training steps on the four contest benchmarks (AIME25, AIME26, HMMT-F25, HMMT-F26) at $k = 4$, and bands are $\pm 1$ standard deviation across seeds.}
  \label{fig:matched-flops}
\end{figure*}

AES combines normalized changes in accuracy and generation
length into a single efficiency value,
\begin{equation}
\label{eq:aes}
  \mathrm{AES} \;=\; \Delta L \;+\; \beta\,\Delta\mathrm{Acc}.
\end{equation}
The length term $\Delta L = (L_{\text{ref}} - L)/L_{\text{ref}}$ is positive when
the method generates shorter responses than the reference.
The accuracy term
$\Delta\mathrm{Acc} = (A - A_{\text{ref}})/A_{\text{ref}}$ is positive when it is
more accurate.
A larger AES is therefore better on both axes.
The weight is $\beta = 3$ when $\Delta\mathrm{Acc} \geq 0$ and $\beta = 5$
otherwise.
This asymmetry penalizes accuracy drops more than it rewards gains.
Because both terms are measured against the reference, the same method scores
differently under different references.
Table~\ref{tab:aes-base} reports two such references.
The top block uses the untrained base model and measures the absolute efficiency
of post-training.
The bottom block uses the no-replay GRPO baseline and measures
the incremental efficiency of replay over no-replay.
GRPO itself sits at zero in the bottom block by construction.
The two blocks therefore answer distinct questions and should not be compared directly across blocks.
Under this metric our method is the highest at every scale under both references.

\vspace{-0.1cm}

\paragraph{Absolute Efficiency (Base Reference)}
GRPO's AES against Base is negative at 0.6B and 4B, where its accuracy gain over Base is offset by length inflation, and our method is above it at all three scales.

\vspace{-0.1cm}

\paragraph{Incremental Efficiency (GRPO Reference)}
Under the GRPO reference, rollout-level $|A_i|$ leads at every scale and is the only condition positive at all three.
The two naive baselines instead straddle zero.
Replay therefore improves token efficiency only under the $|A_i|$ priority.
The 0.6B margin is the smallest of the three, in line with the higher share of all-wrong groups at 0.6B (Appendix~\ref{app:effective-batch}), which leaves fewer mixed-reward groups from which high-$|A_i|$ candidates can be drawn.
The full $12$-configuration AES ranking at 1.7B across the $\alpha$, $\tau_{\max}$, and $r$ sweep is in Appendix~\ref{app:aes-full}.

\vspace{-0.1cm}

\paragraph{Where the Efficiency Gain Comes From}
The AES gain is a redistribution of length rather than a uniform reduction in generation tokens.
Correct rollouts became longer and incorrect ones shorter, so the ratio of mean response length on correct rollouts to that on incorrect ones rises at every scale against the no-replay GRPO baseline, by $+33.6$, $+23.4$, and $+25.2\%$ (Table~\ref{tab:length-quality}, with the two sides in Appendix~\ref{app:aes-references}).
One reading of this follows from what $|A_i|$ selects.
A correct rollout reaches high $|A_i|$ only in a mostly wrong group, and an incorrect rollout only in a mostly correct one.
Replay therefore reinforces the rare correct solutions of hard prompts, which are the long ones, and suppresses the rare failures of easy prompts.
Appendix~\ref{app:kprobe} measures this composition on the untrained base models, where the correct side is the minority in $99\%$ of the 4B groups that carry the highest priority.
We offer this as a conjecture rather than a tested mechanism, since we do not vary the two directions independently.

\subsection{Compute-Matched Efficiency}
\label{sec:dyn-compute}

We additionally analyze the 4B runs at matched compute.
Learner FLOPs are proportional to the trained tokens of each step, so we compare our method and the no-replay baseline at equal trained tokens rather than at equal training steps.
Figure~\ref{fig:matched-flops} plots both conditions over three training seeds.
At the baseline's full budget, marked in the figure, our method leads by $+0.75$~pp on AVG@$K$ and $+2.50$~pp on PASS@$K$.
It continues to its own $200$ steps at $1.29$ times that budget and ends $+2.05$~pp and $+1.59$~pp above the baseline.
Both conditions draw the same fresh rollouts per step and our method trains more tokens per step, so at the marked budget it has seen only $86\%$ of the baseline's training prompts.
Leading on less new data supports the premise behind the buffer, that a single pass leaves learning signal in the rollouts GRPO discards.
Appendices~\ref{app:flops-accounting} and~\ref{app:overhead} report the FLOPs convention, the step-wise curve, and the measured overhead.

\section{Conclusion}
\label{sec:conclusion}

We presented an experience replay buffer for GRPO with two main findings.
First, the naive replay baselines we test give no consistent gain over GRPO across three Qwen3-Base scales of 0.6B, 1.7B, and 4B on DeepScaleR-Preview, and at 4B they fall below it on the five-benchmark average.
Reuse without an appropriate priority can therefore cost accuracy rather than add to it.

Second, our method combines fresh-anchored composition with age eviction and rollout-level $|A_i|$ priority.
It yields a positive margin at every scale, reaching $+1.66$~pp on the five-benchmark average at 4B.
On the contest benchmarks at 4B the pass-side gains reach $+13.3$~pp on AIME26 and $+11.1$~pp on HMMT-F25.

Per-axis ablations support each design choice.
Age eviction bounds policy lag through a step-level cap, and fresh-anchored composition contributes the on-policy anchor.
At 1.7B, rollout-level $|A_i|$ outperforms both uniform and $\sigma_g$ on both metrics.
Each baseline is tested at query and rollout level.
The widest pass gap is $+3.19$~pp over rollout-level $\sigma_g$.
Fresh-anchored composition further adds $+0.58$~pp avg and $+5.27$~pp pass over pool composition.

Beyond accuracy, our method is the only condition with a positive token-budget AES margin over GRPO at every scale.
A correctness-conditioned length decomposition in Section~\ref{sec:dyn-aes} shows the model lengthens correct rollouts relative to incorrect ones, raising their length ratio by $+23$ to $+34\%$ across scales.
These results indicate that rollout-level replay modifies the trained model's response distribution, not only its accuracy.

\section*{Limitations}
\label{sec:limitations}

Our study is confined to math reasoning with binary verifiable rewards.
We use DAPO-style GRPO with asymmetric clipping as the underlying trainer.
We do not study replay under non-verifiable rewards such as learned reward models, nor under other RL algorithms such as PPO or REINFORCE.
Whether the rollout-level $|A_i|$ priority and age eviction generalize to those settings is an open question.

Our update is biased with respect to the on-policy GRPO objective, and Section~\ref{sec:method} states the three sources.
Replayed advantages stay frozen at their birth-step values, the prioritized draw is deliberately left uncorrected because an importance-sampling weight would cancel the reweighting toward high-$|A_i|$ rollouts that motivates replay in the first place, and replayed rollouts carry a policy lag of up to $\tau_{\max}$ steps.
Age eviction, the fresh anchor, and clipping bound these departures rather than remove them.
The $\alpha = 0$ setting of Table~\ref{tab:priority-signal} is the in-expectation endpoint of a fully corrected draw, but we do not run the weighted estimator itself, we do not refresh replayed advantages, and we do not characterize the bias analytically.
Our evidence for the trade-off is empirical.
Refreshing replayed advantages online, in the spirit of online difficulty filtering~\citep{bae2026online, pcl}, is left to future work.

We do not compare against replay-free informative sampling.
Dynamic sampling and oversample-and-filter fill an update with nonzero-advantage groups by generating additional rollouts, so their cost falls on generation, whereas replay adds no generation and pays on the learner instead.
A fair comparison therefore has to be matched on total compute rather than on training steps, which we leave to future work.
The uniform-replay control of Table~\ref{tab:priority-signal} isolates prioritization from informative filtering within replay, but it does not establish whether the same accuracy is reachable by spending the extra budget on generation.

No run is reported beyond global step~200, so longer-horizon dynamics are not characterized.
These include failure modes that may emerge only after substantially more updates and the eventual ceiling of the recycling benefit.
Similarly, the body scaling sweep is limited to Qwen3-Base at 0.6B, 1.7B, and 4B trained on the DeepScaleR-Preview prompt set, and the compute-matched comparison of Section~\ref{sec:dyn-compute} is measured at 4B alone.
Transfer to larger model scales, to multi-turn or agentic tasks, and to non-math domains is not validated.

Cross-scale trends within the tested range are not perfectly monotonic, and the range is bounded on both sides.
The 0.6B model in particular shows smaller margins than 1.7B and 4B on several metrics.
A plausible contributing factor is a scale-dependent shift in prompt difficulty, as suggested by Appendix~\ref{app:effective-batch}.
Only $34\%$ of generated groups are mixed-reward at 0.6B during training and this rises to $49\%$ at 4B.
The rare-correct rollouts that high-$|A_i|$ priority targets are therefore scarce in absolute count at smaller scales.
At the other end, a single 8B run under a reduced training batch shows no gain on the five-benchmark average against its own no-replay baseline.
Appendix~\ref{app:kprobe} reports that run together with a start-of-training solve-count distribution that has moved away from the region $|A_i|$ targets, which is the same misalignment seen from the opposite direction.
RLVR is sensitive to the alignment between prompt difficulty and current policy capability.
Replay may amplify this sensitivity because the buffer keeps recycling high-$|A_i|$ candidates from the same difficulty distribution rather than rebalancing it.
We do not study a remedy in this work.
Difficulty-adaptive prompt selection or buffer composition is a direction for future work.

All evaluation is automatic via a programmatic answer checker.
We do not run human studies on the quality of the responses produced by our method.
The asymmetric length pattern in Section~\ref{sec:dynamics} therefore reflects a verifiable-reward view of efficiency rather than a perceived-quality view.

The ablation tables use a single training seed of $42$ per run.
Evaluation uses a single primary $k = 8$ sampling wave at seed $42$ for every benchmark.
The $k = 8$ wave is kept as the MATH-500 budget, and a second $k = 8$ wave at seed $43$ is added for AIME and HMMT to reach $k = 16$.
The same two seeds are used for every condition at every scale.
Any seed-dependent fluctuation in the per-example samples is therefore shared across the comparisons in the body tables.
The detailed seed policy and the per-wave sample\_id ranges are in Appendix~\ref{app:training-setup} (Training and Sampling Seeds).

\section*{Ethical Considerations}
\label{sec:ethics}

The contribution of this work is a sample-efficient training procedure for GRPO on verifiable-reward math reasoning.
The work is foundational and not tied to a specific deployed application.
Any efficiency improvement in RLVR post-training is dual-use in the sense that it lowers the cost of producing reasoning models.
The method could therefore be applied for both beneficial and harmful purposes by downstream developers.
Misuse scenarios specific to math reasoning include automated solving of homework or examination problem sets, which constitutes academic dishonesty.
They also include large-scale generation of synthetic competition-style solutions, which could contaminate future training corpora if posted publicly without provenance markers.
We release no new model or dataset artifact with this work, which limits direct deployment risk.
See Appendix~\ref{app:training-setup} for the artifacts and licenses used.
The compute footprint of the full sweep is roughly $2{,}500$ H100-hours, as reported in Appendix~\ref{app:training-setup}.
This is modest relative to large-scale post-training pipelines but compounds at scale.
We report it transparently to support reproducibility and discourage redundant re-runs.
We do not study replay under non-verifiable rewards.
No claim is made about transfer to settings with learned reward models or human-feedback signals.

\makeatletter
\ifacl@anonymize\else
\section*{Acknowledgements}

This work was supported by the BK21 FOUR program of the Education and Research Program for Future ICT Pioneers, Seoul National University in 2026, Institute of Information \& Communications Technology Planning \& Evaluation (IITP) grant funded by the Korea government (MSIT) [NO.RS-2021-II211343, Artificial Intelligence Graduate School Program (Seoul National University)], the National Research Foundation of Korea (NRF) grant funded by the Korea government (MSIT) (No.~2022R1A3B1077720), and Institute of Information \& Communications Technology Planning \& Evaluation (IITP) grant funded by the Korea government (MSIT) [NO. RS-2022-II220959].
This research was also conducted as part of the Sovereign AI Foundation Model Project (Data Track), organized by the Ministry of Science and ICT (MSIT) and supported by the National Information Society Agency (NIA), S.~Korea (Grant No.~\mbox{2026-AIData-WII01}).
This work was also supported by the National AI Project of the Ministry of Science and ICT and the AIIS.
\fi
\makeatother

\bibliography{references}

\begin{thebibliography}{85}
\providecommand{\natexlab}[1]{#1}

\bibitem[{Andrychowicz et~al.(2017)Andrychowicz, Wolski, Ray, Schneider, Fong,
  Welinder, McGrew, Tobin, Abbeel, and Zaremba}]{her}
Marcin Andrychowicz, Filip Wolski, Alex Ray, Jonas Schneider, Rachel Fong,
  Peter Welinder, Bob McGrew, Josh Tobin, Pieter Abbeel, and Wojciech Zaremba.
  2017.
\newblock \href
  {https://proceedings.neurips.cc/paper/2017/hash/453fadbd8a1a3af50a9df4df899537b5-Abstract.html}
  {Hindsight experience replay}.
\newblock In \emph{Advances in Neural Information Processing Systems
  ({NeurIPS})}.

\bibitem[{Arnal et~al.(2026)Arnal, Cabannes, Cohen, Kempe, and
  Munos}]{arnal2026efficientrlreplay}
Charles Arnal, Vivien Cabannes, Taco Cohen, Julia Kempe, and R{\'e}mi Munos.
  2026.
\newblock \href {https://arxiv.org/abs/2604.08706} {Efficient {RL} training for
  {LLMs} with experience replay}.
\newblock \emph{arXiv preprint arXiv:2604.08706}.

\bibitem[{Bae et~al.(2026)Bae, Hong, Lee, Kim, Nam, and Kwak}]{bae2026online}
Sanghwan Bae, Jiwoo Hong, Min~Young Lee, Hanbyul Kim, Jeongyeon Nam, and
  Donghyun Kwak. 2026.
\newblock \href {https://doi.org/10.18653/v1/2026.eacl-long.30} {Online
  difficulty filtering for reasoning oriented reinforcement learning}.
\newblock In \emph{Proceedings of the 19th Conference of the {E}uropean Chapter
  of the {A}ssociation for {C}omputational {L}inguistics (Volume 1: Long
  Papers)}, pages 700--719, Rabat, Morocco. Association for Computational
  Linguistics.

\bibitem[{Baroian and Berger(2026)}]{prompt-replay}
Andrei Baroian and Rutger Berger. 2026.
\newblock \href {https://arxiv.org/abs/2603.21177} {Prompt replay: Speeding up
  {GRPO} with on-policy reuse of high-signal prompts}.
\newblock \emph{arXiv preprint arXiv:2603.21177}.

\bibitem[{Bartoldson et~al.(2025)Bartoldson, Venkatraman, Diffenderfer, Jain,
  Ben-Nun, Lee, Kim, Obando-Ceron, Bengio, and Kailkhura}]{tba}
Brian~R. Bartoldson, Siddarth Venkatraman, James Diffenderfer, Moksh Jain, Tal
  Ben-Nun, Seanie Lee, Minsu Kim, Johan Obando-Ceron, Yoshua Bengio, and Bhavya
  Kailkhura. 2025.
\newblock \href {https://arxiv.org/abs/2503.18929} {Trajectory balance with
  asynchrony: Decoupling exploration and learning for fast, scalable {LLM}
  post-training}.
\newblock In \emph{Advances in Neural Information Processing Systems
  ({NeurIPS})}.

\bibitem[{Bellemare et~al.(2017)Bellemare, Dabney, and Munos}]{c51}
Marc~G. Bellemare, Will Dabney, and R{\'e}mi Munos. 2017.
\newblock \href {https://proceedings.mlr.press/v70/bellemare17a.html} {A
  distributional perspective on reinforcement learning}.
\newblock In \emph{Proceedings of the 34th International Conference on Machine
  Learning ({ICML})}, volume~70 of \emph{Proceedings of Machine Learning
  Research}, pages 449--458. {PMLR}.

\bibitem[{Cassirer et~al.(2021)Cassirer, Barth-Maron, Brevdo, Ramos, Boyd,
  Sottiaux, and Kroiss}]{reverb}
Albin Cassirer, Gabriel Barth-Maron, Eugene Brevdo, Sabela Ramos, Toby Boyd,
  Thibault Sottiaux, and Manuel Kroiss. 2021.
\newblock \href {https://arxiv.org/abs/2102.04736} {{Reverb}: A framework for
  experience replay}.
\newblock \emph{arXiv preprint arXiv:2102.04736}.

\bibitem[{Chen et~al.(2021)Chen, Wang, Zhou, and Ross}]{redq}
Xinyue Chen, Che Wang, Zijian Zhou, and Keith~W. Ross. 2021.
\newblock \href {https://openreview.net/forum?id=AY8zfZm0tDd} {Randomized
  ensembled double {Q}-learning: Learning fast without a model}.
\newblock In \emph{International Conference on Learning Representations
  ({ICLR})}.

\bibitem[{Dong et~al.(2023)Dong, Xiong, Goyal, Zhang, Chow, Pan, Diao, Zhang,
  Shum, and Zhang}]{raft}
Hanze Dong, Wei Xiong, Deepanshu Goyal, Yihan Zhang, Winnie Chow, Rui Pan,
  Shizhe Diao, Jipeng Zhang, Kashun Shum, and Tong Zhang. 2023.
\newblock \href {https://openreview.net/forum?id=m7p5O7zblY} {{RAFT}: Reward
  r{A}nked {F}ine-{T}uning for generative foundation model alignment}.
\newblock \emph{Transactions on Machine Learning Research}.

\bibitem[{Dong et~al.(2025)Dong, Jiang, Tao, Liu, Zhang, Mou, Cao, Ma, Chen,
  Li, Jin, Huang, Li, and Li}]{rlplus}
Yihong Dong, Xue Jiang, Yongding Tao, Huanyu Liu, Kechi Zhang, Lili Mou, Rongyu
  Cao, Yingwei Ma, Jue Chen, Binhua Li, Zhi Jin, Fei Huang, Yongbin Li, and
  Ge~Li. 2025.
\newblock \href {https://arxiv.org/abs/2508.00222} {{RL-PLUS}: Countering
  capability boundary collapse of {LLMs} in reinforcement learning with
  hybrid-policy optimization}.
\newblock \emph{arXiv preprint arXiv:2508.00222}.

\bibitem[{D'Oro et~al.(2023)D'Oro, Schwarzer, Nikishin, Bacon, Bellemare, and
  Courville}]{srsac}
Pierluca D'Oro, Max Schwarzer, Evgenii Nikishin, Pierre-Luc Bacon, Marc~G.
  Bellemare, and Aaron Courville. 2023.
\newblock \href {https://openreview.net/forum?id=OpC-9aBBVJe} {Sample-efficient
  reinforcement learning by breaking the replay ratio barrier}.
\newblock In \emph{International Conference on Learning Representations
  ({ICLR})}.

\bibitem[{Fatemi(2026)}]{fatemi2026prioritized}
Mehdi Fatemi. 2026.
\newblock \href {https://arxiv.org/abs/2601.02648} {Prioritized replay for {RL}
  post-training}.
\newblock \emph{arXiv preprint arXiv:2601.02648}.

\bibitem[{Fedus et~al.(2020)Fedus, Ramachandran, Agarwal, Bengio, Larochelle,
  Rowland, and Dabney}]{fedus2020revisit}
William Fedus, Prajit Ramachandran, Rishabh Agarwal, Yoshua Bengio, Hugo
  Larochelle, Mark Rowland, and Will Dabney. 2020.
\newblock \href {https://proceedings.mlr.press/v119/fedus20a.html} {Revisiting
  fundamentals of experience replay}.
\newblock In \emph{Proceedings of the 37th International Conference on Machine
  Learning ({ICML})}, volume 119 of \emph{Proceedings of Machine Learning
  Research}, pages 3061--3071. PMLR.

\bibitem[{Fu et~al.(2025)Fu, Gao, Shen, Zhu, Mei, He, Xu, Wei, Mei, Wang, Yang,
  Yuan, and Wu}]{areal}
Wei Fu, Jiaxuan Gao, Xujie Shen, Chen Zhu, Zhiyu Mei, Chuyi He, Shusheng Xu,
  Guo Wei, Jun Mei, Jiashu Wang, Tongkai Yang, Binhang Yuan, and Yi~Wu. 2025.
\newblock \href {https://arxiv.org/abs/2505.24298} {{AReaL}: A large-scale
  asynchronous reinforcement learning system for language reasoning}.
\newblock \emph{arXiv preprint arXiv:2505.24298}.

\bibitem[{Fujimoto et~al.(2020)Fujimoto, Meger, and Precup}]{lap}
Scott Fujimoto, David Meger, and Doina Precup. 2020.
\newblock \href
  {https://proceedings.neurips.cc/paper/2020/hash/a3bf6e4db673b6449c2f7d13ee6ec9c0-Abstract.html}
  {An equivalence between loss functions and non-uniform sampling in experience
  replay}.
\newblock In \emph{Advances in Neural Information Processing Systems 33
  ({NeurIPS})}, Online. Curran Associates, Inc.

\bibitem[{Gao et~al.(2024)Gao, Tow, Abbasi, Biderman, Black, DiPofi, Foster,
  Golding, Hsu, Le~Noac'h et~al.}]{lmevalharness}
Leo Gao, Jonathan Tow, Baber Abbasi, Stella Biderman, Sid Black, Anthony
  DiPofi, Charles Foster, Laurence Golding, Jeffrey Hsu, Alain Le~Noac'h, and 1
  others. 2024.
\newblock \href {https://doi.org/10.5281/zenodo.12608602} {A framework for
  few-shot language model evaluation}.
\newblock GitHub repository.

\bibitem[{Gao et~al.(2025)Gao, Kim, Sun, Joachims, Wang, Pang, and Tan}]{pcl}
Zhaolin Gao, Joongwon Kim, Wen Sun, Thorsten Joachims, Sid Wang,
  Richard~Yuanzhe Pang, and Liang Tan. 2025.
\newblock \href {https://arxiv.org/abs/2510.01135} {Prompt curriculum learning
  for efficient {LLM} post-training}.
\newblock \emph{arXiv preprint arXiv:2510.01135}.

\bibitem[{Guo et~al.(2025)Guo, Yang, Zhang, Song, Wang, Zhu, Xu, Zhang, Ma, Bi,
  Zhang, Yu, Wu, Wu, Gou, Shao, Li, Gao, Liu, Xue, Wang, Wu, Feng, Lu, Zhao,
  Deng, Ruan, Dai, Chen, Ji, Li, Lin, Dai, Luo, Hao, Chen, Li, Zhang, Xu, Ding,
  Gao, Qu, Li, Guo, Li, Chen, Yuan, Tu, Qiu, Li, Cai, Ni, Liang, Chen, Dong,
  Hu, You, Gao, Guan, Huang, Yu, Wang, Zhang, Zhao, Wang, Zhang, Xu, Xia,
  Zhang, Zhang, Tang, Zhou, Li, Wang, Li, Tian, Huang, Zhang, Wang, Chen, Du,
  Ge, Zhang, Pan, Wang, Chen, Jin, Chen, Lu, Zhou, Chen, Ye, Wang, Yu, Zhou,
  Pan, Li, Zhou, Wu, Yun, Pei, Sun, Wang, Zeng, Liu, Liang, Gao, Yu, Zhang,
  Xiao, An, Liu, Wang, Chen, Nie, Cheng, Liu, Xie, Liu, Yang, Li, Su, Lin, Li,
  Jin, Shen, Chen, Sun, Wang, Song, Zhou, Wang, Shan, Li, Wang, Wei, Zhang, Xu,
  Li, Zhao, Sun, Wang, Yu, Zhang, Shi, Xiong, He, Piao, Wang, Tan, Ma, Liu,
  Guo, Ou, Wang, Gong, Zou, He, Xiong, Luo, You, Liu, Zhou, Zhu, Huang, Li,
  Zheng, Zhu, Ma, Tang, Zha, Yan, Ren, Ren, Sha, Fu, Xu, Xie, Zhang, Hao, Ma,
  Yan, Wu, Gu, Zhu, Liu, Li, Xie, Song, Pan, Huang, Xu, Zhang, and
  Zhang}]{deepseekr1}
Daya Guo, Dejian Yang, Haowei Zhang, Junxiao Song, Peiyi Wang, Qihao Zhu,
  Runxin Xu, Ruoyu Zhang, Shirong Ma, Xiao Bi, Xiaokang Zhang, Xingkai Yu,
  Yu~Wu, Z.~F. Wu, Zhibin Gou, Zhihong Shao, Zhuoshu Li, Ziyi Gao, Aixin Liu,
  and 175 others. 2025.
\newblock \href {https://doi.org/10.1038/s41586-025-09422-z} {{DeepSeek-R1}
  incentivizes reasoning in {LLMs} through reinforcement learning}.
\newblock \emph{Nature}, 645(8081):633--638.

\bibitem[{Han et~al.(2025)Han, You, Wang, Luo, Yang, Shi, Chen, Zhang, Lan,
  Deng, Ji, Liu, Huang, Zhang, Pan, Wang, Huang, Li, and Wu}]{asyncflow}
Zhenyu Han, Ansheng You, Haibo Wang, Kui Luo, Guang Yang, Wenqi Shi, Menglong
  Chen, Sicheng Zhang, Zeshun Lan, Chunshi Deng, Huazhong Ji, Wenjie Liu,
  Yu~Huang, Yixiang Zhang, Chenyi Pan, Jing Wang, Xin Huang, Chunsheng Li, and
  Jianping Wu. 2025.
\newblock \href {https://arxiv.org/abs/2507.01663} {{AsyncFlow}: An
  asynchronous streaming {RL} framework for efficient {LLM} post-training}.
\newblock \emph{arXiv preprint arXiv:2507.01663}.

\bibitem[{{Harvard--MIT Mathematics Tournament}(2025)}]{hmmt_feb2025}
{Harvard--MIT Mathematics Tournament}. 2025.
\newblock \href {https://www.hmmt.org} {{HMMT} february 2025}.
\newblock Harvard--MIT Mathematics Tournament, February 2025.

\bibitem[{{Harvard--MIT Mathematics Tournament}(2026)}]{hmmt_feb2026}
{Harvard--MIT Mathematics Tournament}. 2026.
\newblock \href {https://www.hmmt.org} {{HMMT} february 2026}.
\newblock Harvard--MIT Mathematics Tournament, February 2026.

\bibitem[{Hessel et~al.(2018)Hessel, Modayil, van Hasselt, Schaul, Ostrovski,
  Dabney, Horgan, Piot, Azar, and Silver}]{rainbow}
Matteo Hessel, Joseph Modayil, Hado van Hasselt, Tom Schaul, Georg Ostrovski,
  Will Dabney, Dan Horgan, Bilal Piot, Mohammad Azar, and David Silver. 2018.
\newblock \href {https://doi.org/10.1609/aaai.v32i1.11796} {Rainbow: Combining
  improvements in deep reinforcement learning}.
\newblock In \emph{Proceedings of the {AAAI} Conference on Artificial
  Intelligence}, volume~32, pages 3215--3222. {AAAI} Press.

\bibitem[{Hong et~al.(2022)Hong, Chen, Lin, Pajarinen, and
  Agrawal}]{topological-replay}
Zhang-Wei Hong, Tao Chen, Yen-Chen Lin, Joni Pajarinen, and Pulkit Agrawal.
  2022.
\newblock \href {https://openreview.net/forum?id=OXRZeMmOI7a} {Topological
  experience replay}.
\newblock In \emph{International Conference on Learning Representations
  ({ICLR})}.

\bibitem[{Horgan et~al.(2018)Horgan, Quan, Budden, Barth-Maron, Hessel, van
  Hasselt, and Silver}]{apex}
Dan Horgan, John Quan, David Budden, Gabriel Barth-Maron, Matteo Hessel, Hado
  van Hasselt, and David Silver. 2018.
\newblock \href {https://arxiv.org/abs/1803.00933} {Distributed prioritized
  experience replay}.
\newblock In \emph{International Conference on Learning Representations
  ({ICLR})}.

\bibitem[{Hu et~al.(2024)Hu, Wu, Shen, Liu, Zhu, Wang, Jiang, Wang, Chen, Chen,
  Fang, Xianyu, Cao, Xu, and Liu}]{openrlhf}
Jian Hu, Xibin Wu, Wei Shen, Jason~Klein Liu, Zilin Zhu, Weixun Wang, Songlin
  Jiang, Haoran Wang, Hao Chen, Bin Chen, Weikai Fang, Xianyu, Yu~Cao, Haotian
  Xu, and Yiming Liu. 2024.
\newblock \href {https://arxiv.org/abs/2405.11143} {{OpenRLHF}: An easy-to-use,
  scalable and high-performance {RLHF} framework}.
\newblock \emph{arXiv preprint arXiv:2405.11143}.

\bibitem[{{Hugging Face}(2025)}]{mathverify}
{Hugging Face}. 2025.
\newblock \href {https://github.com/huggingface/Math-Verify} {Math-verify:
  Mathematical answer verification}.
\newblock GitHub repository.

\bibitem[{Isele and Cosgun(2018)}]{selective-replay}
David Isele and Akansel Cosgun. 2018.
\newblock \href {https://doi.org/10.1609/aaai.v32i1.11595} {Selective
  experience replay for lifelong learning}.
\newblock In \emph{Proceedings of the {AAAI} Conference on Artificial
  Intelligence}, volume~32. {AAAI} Press.

\bibitem[{Jiang et~al.(2025)Jiang, Feng, Quan, Hao, Zhang, Liu, and
  Wang}]{vcrl}
Guochao Jiang, Wenfeng Feng, Guofeng Quan, Chuzhan Hao, Yuewei Zhang, Guohua
  Liu, and Hao Wang. 2025.
\newblock \href {https://arxiv.org/abs/2509.19803} {{VCRL}: Variance-based
  curriculum reinforcement learning for large language models}.
\newblock \emph{arXiv preprint arXiv:2509.19803}.

\bibitem[{Kaplan et~al.(2020)Kaplan, McCandlish, Henighan, Brown, Chess, Child,
  Gray, Radford, Wu, and Amodei}]{kaplan2020scaling}
Jared Kaplan, Sam McCandlish, Tom Henighan, Tom~B. Brown, Benjamin Chess, Rewon
  Child, Scott Gray, Alec Radford, Jeffrey Wu, and Dario Amodei. 2020.
\newblock \href {https://arxiv.org/abs/2001.08361} {Scaling laws for neural
  language models}.
\newblock \emph{arXiv preprint arXiv:2001.08361}.

\bibitem[{Kapturowski et~al.(2019)Kapturowski, Ostrovski, Quan, Munos, and
  Dabney}]{r2d2}
Steven Kapturowski, Georg Ostrovski, John Quan, R{\'e}mi Munos, and Will
  Dabney. 2019.
\newblock \href {https://openreview.net/forum?id=r1lyTjAqYX} {Recurrent
  experience replay in distributed reinforcement learning}.
\newblock In \emph{International Conference on Learning Representations
  ({ICLR})}.

\bibitem[{Kwon et~al.(2023)Kwon, Li, Zhuang, Sheng, Zheng, Yu, Gonzalez, Zhang,
  and Stoica}]{vllm}
Woosuk Kwon, Zhuohan Li, Siyuan Zhuang, Ying Sheng, Lianmin Zheng, Cody~Hao Yu,
  Joseph~E. Gonzalez, Hao Zhang, and Ion Stoica. 2023.
\newblock \href {https://doi.org/10.1145/3600006.3613165} {Efficient memory
  management for large language model serving with {PagedAttention}}.
\newblock In \emph{Proceedings of the 29th Symposium on Operating Systems
  Principles ({SOSP})}.

\bibitem[{Le et~al.(2026)Le, Jeon, Vu, Lai, and Yang}]{no-prompt-left-behind}
Thanh-Long~V. Le, Myeongho Jeon, Kim Vu, Viet~Dac Lai, and Eunho Yang. 2026.
\newblock \href {https://openreview.net/forum?id=kiXFIESZKv} {No prompt left
  behind: Exploiting zero-variance prompts in {LLM} reinforcement learning via
  entropy-guided advantage shaping}.
\newblock In \emph{The Fourteenth International Conference on Learning
  Representations ({ICLR})}.

\bibitem[{Le~Roux et~al.(2025)Le~Roux, Bellemare, Lebensold, Bergeron, Greaves,
  Fr{\'e}chette, Pelletier, Thibodeau-Laufer, Toth, and Work}]{topr}
Nicolas Le~Roux, Marc~G. Bellemare, Jonathan Lebensold, Arnaud Bergeron, Joshua
  Greaves, Alex Fr{\'e}chette, Carolyne Pelletier, Eric Thibodeau-Laufer,
  S{\'a}ndor Toth, and Sam Work. 2025.
\newblock \href {https://arxiv.org/abs/2503.14286} {Tapered off-policy
  {REINFORCE}: Stable and efficient reinforcement learning for large language
  models}.
\newblock In \emph{Advances in Neural Information Processing Systems
  ({NeurIPS})}.

\bibitem[{Li et~al.(2026)Li, Zhou, Wang, Xu, Huang, Chu, Wang, Pan, Qu, and
  Qi}]{dyjr}
Long Li, Zhijian Zhou, Tianyi Wang, Weidi Xu, Zuming Huang, Wei Chu, Zhe Wang,
  Shirui Pan, Chao Qu, and Yuan Qi. 2026.
\newblock \href {https://arxiv.org/abs/2603.16157} {{DyJR}: Preserving
  diversity in reinforcement learning with verifiable rewards via dynamic
  {Jensen-Shannon} replay}.
\newblock \emph{arXiv preprint arXiv:2603.16157}.

\bibitem[{Li et~al.(2025)Li, Zhou, Lam, Yang, and Lu}]{repo}
Siheng Li, Zhanhui Zhou, Wai Lam, Chao Yang, and Chaochao Lu. 2025.
\newblock \href {https://arxiv.org/abs/2506.09340} {{RePO}: Replay-enhanced
  policy optimization}.
\newblock \emph{arXiv preprint arXiv:2506.09340}.

\bibitem[{Lightman et~al.(2024)Lightman, Kosaraju, Burda, Edwards, Baker, Lee,
  Leike, Schulman, Sutskever, and Cobbe}]{lightman2023lets}
Hunter Lightman, Vineet Kosaraju, Yuri Burda, Harrison Edwards, Bowen Baker,
  Teddy Lee, Jan Leike, John Schulman, Ilya Sutskever, and Karl Cobbe. 2024.
\newblock \href {https://openreview.net/forum?id=v8L0pN6EOi} {Let's verify step
  by step}.
\newblock In \emph{The Twelfth International Conference on Learning
  Representations ({ICLR})}.

\bibitem[{Lillicrap et~al.(2016)Lillicrap, Hunt, Pritzel, Heess, Erez, Tassa,
  Silver, and Wierstra}]{ddpg}
Timothy~P. Lillicrap, Jonathan~J. Hunt, Alexander Pritzel, Nicolas Heess, Tom
  Erez, Yuval Tassa, David Silver, and Daan Wierstra. 2016.
\newblock \href {https://arxiv.org/abs/1509.02971} {Continuous control with
  deep reinforcement learning}.
\newblock In \emph{International Conference on Learning Representations
  ({ICLR})}.

\bibitem[{Lin(1992)}]{lin1992self}
Long-Ji Lin. 1992.
\newblock \href {https://doi.org/10.1007/BF00992699} {Self-improving reactive
  agents based on reinforcement learning, planning and teaching}.
\newblock \emph{Machine Learning}, 8(3--4):293--321.

\bibitem[{Lin et~al.(2025)Lin, Lin, Xie, and Ji}]{cppo}
Zhihang Lin, Mingbao Lin, Yuan Xie, and Rongrong Ji. 2025.
\newblock \href {https://openreview.net/forum?id=SVHerutWxp} {{CPPO}:
  Accelerating the training of group relative policy optimization-based
  reasoning models}.
\newblock In \emph{Advances in Neural Information Processing Systems
  ({NeurIPS})}.

\bibitem[{Liu et~al.(2025{\natexlab{a}})Liu, Si, Narasimhan, and
  Yao}]{cer-language}
Yitao Liu, Chenglei Si, Karthik~R. Narasimhan, and Shunyu Yao.
  2025{\natexlab{a}}.
\newblock \href {https://doi.org/10.18653/v1/2025.acl-long.694} {Contextual
  experience replay for self-improvement of language agents}.
\newblock In \emph{Proceedings of the 63rd Annual Meeting of the Association
  for Computational Linguistics (Volume 1: Long Papers)}, pages 14179--14198,
  Vienna, Austria. Association for Computational Linguistics.

\bibitem[{Liu et~al.(2025{\natexlab{b}})Liu, Wang, Song, and Bian}]{lorr}
Zichuan Liu, Jinyu Wang, Lei Song, and Jiang Bian. 2025{\natexlab{b}}.
\newblock \href {https://arxiv.org/abs/2508.06412} {Sample-efficient {LLM}
  optimization with reset replay}.
\newblock \emph{arXiv preprint arXiv:2508.06412}.

\bibitem[{Luo et~al.(2025{\natexlab{a}})Luo, Shen, He, Wang, Liu, Li, Tan, Cao,
  and Tao}]{o1-pruner}
Haotian Luo, Li~Shen, Haiying He, Yibo Wang, Shiwei Liu, Wei Li, Naiqiang Tan,
  Xiaochun Cao, and Dacheng Tao. 2025{\natexlab{a}}.
\newblock \href {https://arxiv.org/abs/2501.12570} {{O1-Pruner}:
  Length-harmonizing fine-tuning for {O1}-like reasoning pruning}.
\newblock \emph{arXiv preprint arXiv:2501.12570}.

\bibitem[{Luo et~al.(2025{\natexlab{b}})Luo, Tan, Wong, Shi, Tang, Roongta,
  Cai, Luo, Zhang, Li, Popa, and Stoica}]{deepscaler}
Michael Luo, Sijun Tan, Justin Wong, Xiaoxiang Shi, William~Y. Tang, Manan
  Roongta, Colin Cai, Jeffrey Luo, Tianjun Zhang, Li~Erran Li, Raluca~Ada Popa,
  and Ion Stoica. 2025{\natexlab{b}}.
\newblock {DeepScaleR}: Surpassing {O1-Preview} with a 1.5{B} model by scaling
  {RL}.
\newblock Notion Blog.
\newblock
  \url{https://pretty-radio-b75.notion.site/DeepScaleR-Surpassing-O1-Preview-with-a-1-5B-Model-by-Scaling-RL-19681902c1468005bed8ca303013a4e2}.

\bibitem[{Ma et~al.(2026)Ma, Zeng, Song, Cui, Zhao, Liu, and
  Elhoseiny}]{freshper}
Weiyu Ma, Yongcheng Zeng, Yan Song, Xinyu Cui, Jian Zhao, Xuhui Liu, and
  Mohamed Elhoseiny. 2026.
\newblock \href {https://arxiv.org/abs/2604.16918} {Freshness-aware prioritized
  experience replay for {LLM}/{VLM} reinforcement learning}.
\newblock \emph{arXiv preprint arXiv:2604.16918}.

\bibitem[{Mao et~al.(2025)Mao, Xiao, Pang, and Liu}]{fspo}
Hanyi Mao, Quanjia Xiao, Lei Pang, and Haixiao Liu. 2025.
\newblock \href {https://arxiv.org/abs/2509.09177} {Clip your sequences fairly:
  Enforcing length fairness for sequence-level {RL}}.
\newblock \emph{arXiv preprint arXiv:2509.09177}.

\bibitem[{{Mathematical Association of America}(2025)}]{aime2025}
{Mathematical Association of America}. 2025.
\newblock \href {https://artofproblemsolving.com/wiki/index.php/2025_AIME_I}
  {{AIME} 2025 problems}.
\newblock American Invitational Mathematics Examination, February 2025.

\bibitem[{{Mathematical Association of America}(2026)}]{aime2026}
{Mathematical Association of America}. 2026.
\newblock \href {https://artofproblemsolving.com/wiki/index.php/2026_AIME_I}
  {{AIME} 2026 problems}.
\newblock American Invitational Mathematics Examination, February 2026.

\bibitem[{{MiniMax}(2025)}]{minimaxm1}
{MiniMax}. 2025.
\newblock \href {https://arxiv.org/abs/2506.13585} {{MiniMax-M1}: Scaling
  test-time compute efficiently with lightning attention}.
\newblock \emph{arXiv preprint arXiv:2506.13585}.

\bibitem[{Mnih et~al.(2015)Mnih, Kavukcuoglu, Silver, Rusu, Veness, Bellemare,
  Graves, Riedmiller, Fidjeland, Ostrovski, Petersen, Beattie, Sadik,
  Antonoglou, King, Kumaran, Wierstra, Legg, and Hassabis}]{dqn}
Volodymyr Mnih, Koray Kavukcuoglu, David Silver, Andrei~A. Rusu, Joel Veness,
  Marc~G. Bellemare, Alex Graves, Martin Riedmiller, Andreas~K. Fidjeland,
  Georg Ostrovski, Stig Petersen, Charles Beattie, Amir Sadik, Ioannis
  Antonoglou, Helen King, Dharshan Kumaran, Daan Wierstra, Shane Legg, and
  Demis Hassabis. 2015.
\newblock \href {https://doi.org/10.1038/nature14236} {Human-level control
  through deep reinforcement learning}.
\newblock \emph{Nature}, 518(7540):529--533.

\bibitem[{Nikishin et~al.(2022)Nikishin, Schwarzer, D'Oro, Bacon, and
  Courville}]{primacy-bias}
Evgenii Nikishin, Max Schwarzer, Pierluca D'Oro, Pierre-Luc Bacon, and Aaron
  Courville. 2022.
\newblock \href {https://proceedings.mlr.press/v162/nikishin22a.html} {The
  primacy bias in deep reinforcement learning}.
\newblock In \emph{Proceedings of the 39th International Conference on Machine
  Learning ({ICML})}, volume 162 of \emph{Proceedings of Machine Learning
  Research}, pages 16828--16847. PMLR.

\bibitem[{Noukhovitch et~al.(2025)Noukhovitch, Huang, Xhonneux, Hosseini,
  Agarwal, and Courville}]{async-rlhf}
Michael Noukhovitch, Shengyi Huang, Sophie Xhonneux, Arian Hosseini, Rishabh
  Agarwal, and Aaron Courville. 2025.
\newblock \href {https://openreview.net/forum?id=FhTAG591Ve} {Asynchronous
  {RLHF}: Faster and more efficient off-policy {RL} for language models}.
\newblock In \emph{The Thirteenth International Conference on Learning
  Representations ({ICLR})}.

\bibitem[{Novati and Koumoutsakos(2019)}]{refer}
Guido Novati and Petros Koumoutsakos. 2019.
\newblock \href {https://proceedings.mlr.press/v97/novati19a.html} {Remember
  and forget for experience replay}.
\newblock In \emph{Proceedings of the 36th International Conference on Machine
  Learning ({ICML})}, volume~97 of \emph{Proceedings of Machine Learning
  Research}, pages 4851--4860. {PMLR}.

\bibitem[{Razin et~al.(2024)Razin, Zhou, Saremi, Thilak, Bradley, Nakkiran,
  Susskind, and Littwin}]{razin2024}
Noam Razin, Hattie Zhou, Omid Saremi, Vimal Thilak, Arwen Bradley, Preetum
  Nakkiran, Joshua~M. Susskind, and Etai Littwin. 2024.
\newblock \href {https://openreview.net/forum?id=IcVNBR7qZi} {Vanishing
  gradients in reinforcement finetuning of language models}.
\newblock In \emph{The Twelfth International Conference on Learning
  Representations ({ICLR})}.

\bibitem[{Rolnick et~al.(2019)Rolnick, Ahuja, Schwarz, Lillicrap, and
  Wayne}]{clear}
David Rolnick, Arun Ahuja, Jonathan Schwarz, Timothy~P. Lillicrap, and Greg
  Wayne. 2019.
\newblock \href
  {https://proceedings.neurips.cc/paper/2019/hash/fa7cdfad1a5aaf8370ebeda47a1ff1c3-Abstract.html}
  {Experience replay for continual learning}.
\newblock In \emph{Advances in Neural Information Processing Systems 32
  ({NeurIPS})}, Vancouver, Canada. Curran Associates, Inc.

\bibitem[{Schaul et~al.(2016)Schaul, Quan, Antonoglou, and Silver}]{per}
Tom Schaul, John Quan, Ioannis Antonoglou, and David Silver. 2016.
\newblock \href {https://arxiv.org/abs/1511.05952} {Prioritized experience
  replay}.
\newblock In \emph{International Conference on Learning Representations
  ({ICLR})}.

\bibitem[{Shao et~al.(2024)Shao, Wang, Zhu, Xu, Song, Bi, Zhang, Zhang, Li, Wu,
  and Guo}]{deepseekmath}
Zhihong Shao, Peiyi Wang, Qihao Zhu, Runxin Xu, Junxiao Song, Xiao Bi, Haowei
  Zhang, Mingchuan Zhang, Y.~K. Li, Y.~Wu, and Daya Guo. 2024.
\newblock \href {https://arxiv.org/abs/2402.03300} {{DeepSeekMath}: Pushing the
  limits of mathematical reasoning in open language models}.
\newblock \emph{arXiv preprint arXiv:2402.03300}.

\bibitem[{Sheng et~al.(2024)Sheng, Zhang, Ye, Wu, Zhang, Zhang, Peng, Lin, and
  Wu}]{verl}
Guangming Sheng, Chi Zhang, Zilingfeng Ye, Xibin Wu, Wang Zhang, Ru~Zhang,
  Yanghua Peng, Haibin Lin, and Chuan Wu. 2024.
\newblock \href {https://arxiv.org/abs/2409.19256} {{HybridFlow}: A flexible
  and efficient {RLHF} framework}.
\newblock \emph{arXiv preprint arXiv:2409.19256}.

\bibitem[{Shi et~al.(2025)Shi, Wu, Song, Zhou, and Zhao}]{adarft}
Taiwei Shi, Yiyang Wu, Linxin Song, Tianyi Zhou, and Jieyu Zhao. 2025.
\newblock \href {https://arxiv.org/abs/2504.05520} {Efficient reinforcement
  finetuning via adaptive curriculum learning}.
\newblock \emph{arXiv preprint arXiv:2504.05520}.

\bibitem[{Singh et~al.(2024)Singh, Co-Reyes, Agarwal, Anand, Patil, Garcia,
  Liu, Harrison, Lee, Xu, Parisi, Kumar, Alemi, Rizkowsky, Nova, Adlam, Bohnet,
  Elsayed, Sedghi, Mordatch, Simpson, Gur, Snoek, Pennington, Hron, Kenealy,
  Swersky, Mahajan, Culp, Xiao, Bileschi, Constant, Novak, Liu, Warkentin,
  Qian, Bansal, Dyer, Neyshabur, Sohl-Dickstein, and Fiedel}]{restem}
Avi Singh, John~D. Co-Reyes, Rishabh Agarwal, Ankesh Anand, Piyush Patil,
  Xavier Garcia, Peter~J. Liu, James Harrison, Jaehoon Lee, Kelvin Xu, Aaron
  Parisi, Abhishek Kumar, Alex Alemi, Alex Rizkowsky, Azade Nova, Ben Adlam,
  Bernd Bohnet, Gamaleldin Elsayed, Hanie Sedghi, and 22 others. 2024.
\newblock \href {https://openreview.net/forum?id=lNAyUngGFK} {Beyond human
  data: Scaling self-training for problem-solving with language models}.
\newblock \emph{Transactions on Machine Learning Research}.

\bibitem[{Sinha et~al.(2022)Sinha, Song, Garg, and Ermon}]{lfiw}
Samarth Sinha, Jiaming Song, Animesh Garg, and Stefano Ermon. 2022.
\newblock \href {https://proceedings.mlr.press/v168/sinha22a.html} {Experience
  replay with likelihood-free importance weights}.
\newblock In \emph{Proceedings of the 4th Annual Learning for Dynamics and
  Control Conference ({L4DC})}, volume 168 of \emph{Proceedings of Machine
  Learning Research}, pages 110--123. PMLR.

\bibitem[{Sun et~al.(2025)Sun, Shen, Wang, Chen, Wang, Zhou, and Zhang}]{dots}
Yifan Sun, Jingyan Shen, Yibin Wang, Tianyu Chen, Zhendong Wang, Mingyuan Zhou,
  and Huan Zhang. 2025.
\newblock \href {https://openreview.net/forum?id=uwUkETPIJN} {Improving data
  efficiency for {LLM} reinforcement fine-tuning through difficulty-targeted
  online data selection and rollout replay}.
\newblock In \emph{Advances in Neural Information Processing Systems
  ({NeurIPS})}.

\bibitem[{van Hasselt et~al.(2018)van Hasselt, Doron, Strub, Hessel, Sonnerat,
  and Modayil}]{deadly-triad}
Hado van Hasselt, Yotam Doron, Florian Strub, Matteo Hessel, Nicolas Sonnerat,
  and Joseph Modayil. 2018.
\newblock \href {https://arxiv.org/abs/1812.02648} {Deep reinforcement learning
  and the deadly triad}.
\newblock \emph{arXiv preprint arXiv:1812.02648}.

\bibitem[{Wan et~al.(2026)Wan, Wang, Huang, and Sun}]{buffer-matters}
Xu~Wan, Yansheng Wang, Wenqi Huang, and Mingyang Sun. 2026.
\newblock \href {https://arxiv.org/abs/2602.20722} {Buffer matters: Unleashing
  the power of off-policy reinforcement learning in large language model
  reasoning}.
\newblock \emph{arXiv preprint arXiv:2602.20722}.
\newblock Introduces ``Batch Adaptation Policy Optimization (BAPO)''; the
  {BAPO} acronym here is distinct from the {BAPO} (Balanced Policy Optimization
  with Adaptive Clipping) of \citet{bapo-clipping}.

\bibitem[{Wang et~al.(2025{\natexlab{a}})Wang, Wei, Zhang, Shao, Dan, Huang,
  Zhang, and Wang}]{efframe}
Chen Wang, Lai Wei, Yanzhi Zhang, Chenyang Shao, Zedong Dan, Weiran Huang,
  Yuzhi Zhang, and Yue Wang. 2025{\natexlab{a}}.
\newblock \href {https://arxiv.org/abs/2506.22200} {{EFRame}: Deeper reasoning
  via exploration-filter-replay reinforcement learning framework}.
\newblock \emph{arXiv preprint arXiv:2506.22200}.

\bibitem[{Wang et~al.(2025{\natexlab{b}})Wang, Cui, Li, Wan, and Zhao}]{dump}
Zhenting Wang, Guofeng Cui, Yu-Jhe Li, Kun Wan, and Wentian Zhao.
  2025{\natexlab{b}}.
\newblock \href {https://arxiv.org/abs/2504.09710} {{DUMP}: Automated
  distribution-level curriculum learning for {RL}-based {LLM} post-training}.
\newblock \emph{arXiv preprint arXiv:2504.09710}.

\bibitem[{Wu et~al.(2025)Wu, Wang, Tang, Ding, Helenowski, Tan, Xu, Gowda,
  Chen, Zhu, Tang, Qian, Zhu, and Hou}]{llamarl}
Bo~Wu, Sid Wang, Yunhao Tang, Jia Ding, Eryk Helenowski, Liang Tan, Tengyu Xu,
  Tushar Gowda, Zhengxing Chen, Chen Zhu, Xiaocheng Tang, Yundi Qian, Beibei
  Zhu, and Rui Hou. 2025.
\newblock \href {https://arxiv.org/abs/2505.24034} {{LlamaRL}: A distributed
  asynchronous reinforcement learning framework for efficient large-scale {LLM}
  training}.
\newblock \emph{arXiv preprint arXiv:2505.24034}.

\bibitem[{Xi et~al.(2025)Xi, Guo, Nan, Zhou, Shen, Chen, Liu, Huang, Zhang,
  Guo, Deng, Lei, Zheng, Wang, Zhang, Sun, Zheng, Yan, Gui, Zhang, and
  Huang}]{bapo-clipping}
Zhiheng Xi, Xin Guo, Yang Nan, Enyu Zhou, Junrui Shen, Wenxiang Chen, Jiaqi
  Liu, Jixuan Huang, Zhihao Zhang, Honglin Guo, Xun Deng, Zhikai Lei, Miao
  Zheng, Guoteng Wang, Shuo Zhang, Peng Sun, Rui Zheng, Hang Yan, Tao Gui, and
  2 others. 2025.
\newblock \href {https://arxiv.org/abs/2510.18927} {{BAPO}: Stabilizing
  off-policy reinforcement learning for {LLMs} via balanced policy optimization
  with adaptive clipping}.
\newblock \emph{arXiv preprint arXiv:2510.18927}.

\bibitem[{Xiong et~al.(2025{\natexlab{a}})Xiong, Yao, Xu, Pang, Wang, Sahoo,
  Li, Jiang, Zhang, Xiong, and Dong}]{raftrej}
Wei Xiong, Jiarui Yao, Yuhui Xu, Bo~Pang, Lei Wang, Doyen Sahoo, Junnan Li, Nan
  Jiang, Tong Zhang, Caiming Xiong, and Hanze Dong. 2025{\natexlab{a}}.
\newblock \href {https://arxiv.org/abs/2504.11343} {A minimalist approach to
  {LLM} reasoning: from rejection sampling to reinforce}.
\newblock \emph{arXiv preprint arXiv:2504.11343}.

\bibitem[{Xiong et~al.(2025{\natexlab{b}})Xiong, Ye, Liao, Dong, Xu, Monz,
  Bian, Jiang, and Zhang}]{reinforce-ada}
Wei Xiong, Chenlu Ye, Baohao Liao, Hanze Dong, Xinxing Xu, Christof Monz, Jiang
  Bian, Nan Jiang, and Tong Zhang. 2025{\natexlab{b}}.
\newblock \href {https://arxiv.org/abs/2510.04996} {{Reinforce-Ada}: An
  adaptive sampling framework under non-linear {RL} objectives}.
\newblock \emph{arXiv preprint arXiv:2510.04996}.

\bibitem[{Xu et~al.(2025)Xu, Savani, Fang, and Kolter}]{pods}
Yixuan~Even Xu, Yash Savani, Fei Fang, and J.~Zico Kolter. 2025.
\newblock \href {https://arxiv.org/abs/2504.13818} {Not all rollouts are
  useful: Down-sampling rollouts in {LLM} reinforcement learning}.
\newblock \emph{arXiv preprint arXiv:2504.13818}.

\bibitem[{Yan et~al.(2025)Yan, Li, Hu, Wang, Cui, Qu, Cheng, and Zhang}]{luffy}
Jianhao Yan, Yafu Li, Zican Hu, Zhi Wang, Ganqu Cui, Xiaoye Qu, Yu~Cheng, and
  Yue Zhang. 2025.
\newblock \href {https://openreview.net/forum?id=vO8LLoNWWk} {Learning to
  reason under off-policy guidance}.
\newblock In \emph{Advances in Neural Information Processing Systems
  ({NeurIPS})}.

\bibitem[{Yang et~al.(2025)Yang, Li, Yang, Zhang, Hui, Zheng, Yu, Gao, Huang,
  Lv, Zheng, Liu, Zhou, Huang, Hu, Ge, Wei, Lin, Tang, Yang, Tu, Zhang, Yang,
  Yang, Zhou, Zhou, Lin, Dang, Bao, Yang, Yu, Deng, Li, Xue, Li, Zhang, Wang,
  Zhu, Men, Gao, Liu, Luo, Li, Tang, Yin, Ren, Wang, Zhang, Ren, Fan, Su,
  Zhang, Zhang, Wan, Liu, Wang, Cui, Zhang, Zhou, and Qiu}]{qwen3}
An~Yang, Anfeng Li, Baosong Yang, Beichen Zhang, Binyuan Hui, Bo~Zheng, Bowen
  Yu, Chang Gao, Chengen Huang, Chenxu Lv, Chujie Zheng, Dayiheng Liu, Fan
  Zhou, Fei Huang, Feng Hu, Hao Ge, Haoran Wei, Huan Lin, Jialong Tang, and 41
  others. 2025.
\newblock \href {https://arxiv.org/abs/2505.09388} {{Qwen3} technical report}.
\newblock \emph{arXiv preprint arXiv:2505.09388}.

\bibitem[{Yao et~al.(2025)Yao, Hao, Zhang, Dong, Xiong, Jiang, and
  Zhang}]{gvmraft}
Jiarui Yao, Yifan Hao, Hanning Zhang, Hanze Dong, Wei Xiong, Nan Jiang, and
  Tong Zhang. 2025.
\newblock \href {https://arxiv.org/abs/2505.02391} {Optimizing chain-of-thought
  reasoners via gradient variance minimization in rejection sampling and {RL}}.
\newblock \emph{arXiv preprint arXiv:2505.02391}.

\bibitem[{Yu et~al.(2025)Yu, Zhang, Zhu, Yuan, Zuo, Yue, Dai, Fan, Liu, Liu,
  Liu, Liu, Lin, Lin, Ma, Sheng, Tong, Zhang, Zhang, Zhang, Zhang, Zhu, Zhu,
  Chen, Chen, Wang, Yu, Song, Wei, Zhou, Liu, Ma, Zhang, Yan, Wu, and
  Wang}]{dapo}
Qiying Yu, Zheng Zhang, Ruofei Zhu, Yufeng Yuan, Xiaochen Zuo, Yu~Yue, Weinan
  Dai, Tiantian Fan, Gaohong Liu, Juncai Liu, Lingjun Liu, Xin Liu, Haibin Lin,
  Zhiqi Lin, Bole Ma, Guangming Sheng, Yuxuan Tong, Chi Zhang, Mofan Zhang, and
  17 others. 2025.
\newblock \href
  {https://proceedings.neurips.cc/paper_files/paper/2025/hash/a4277440d50f1f15d2cb4c14f7e0c0d2-Abstract-Conference.html}
  {{DAPO}: An open-source {LLM} reinforcement learning system at scale}.
\newblock In \emph{Advances in Neural Information Processing Systems 38
  ({NeurIPS} 2025)}.

\bibitem[{Zhan et~al.(2026)Zhan, Li, Wang, Qu, Liu, Shao, Wong, and
  Cheng}]{exgrpo}
Runzhe Zhan, Yafu Li, Zhi Wang, Xiaoye Qu, Dongrui Liu, Jing Shao, Derek~F.
  Wong, and Yu~Cheng. 2026.
\newblock \href {https://arxiv.org/abs/2510.02245} {{ExGRPO}: Learning to
  reason from experience}.
\newblock In \emph{Proceedings of the International Conference on Learning
  Representations ({ICLR})}.

\bibitem[{Zhang et~al.(2025{\natexlab{a}})Zhang, Fu, Zhang, Fu, Wang, Zhang,
  and Zhou}]{rlep}
Hongzhi Zhang, Jia Fu, Jingyuan Zhang, Kai Fu, Qi~Wang, Fuzheng Zhang, and
  Guorui Zhou. 2025{\natexlab{a}}.
\newblock \href {https://arxiv.org/abs/2507.07451} {{RLEP}: Reinforcement
  learning with experience replay for {LLM} reasoning}.
\newblock \emph{arXiv preprint arXiv:2507.07451}.

\bibitem[{Zhang et~al.(2025{\natexlab{b}})Zhang, Arora, Mei, and
  Zanette}]{speedrl}
Ruiqi Zhang, Daman Arora, Song Mei, and Andrea Zanette. 2025{\natexlab{b}}.
\newblock \href {https://arxiv.org/abs/2506.09016} {{SPEED-RL}: Faster training
  of reasoning models via online curriculum learning}.
\newblock \emph{arXiv preprint arXiv:2506.09016}.

\bibitem[{Zhang and Sutton(2017)}]{cer}
Shangtong Zhang and Richard~S. Sutton. 2017.
\newblock \href {https://arxiv.org/abs/1712.01275} {A deeper look at experience
  replay}.
\newblock \emph{arXiv preprint arXiv:1712.01275}.

\bibitem[{Zhang et~al.(2025{\natexlab{c}})Zhang, Yao, Yu, Liu, Yin, Yin, Yun,
  and Li}]{ar3po}
Yuheng Zhang, Wenlin Yao, Changlong Yu, Yao Liu, Qingyu Yin, Bing Yin, Hyokun
  Yun, and Lihong Li. 2025{\natexlab{c}}.
\newblock \href {https://arxiv.org/abs/2509.25808} {Improving sampling
  efficiency in {RLVR} through adaptive rollout and response reuse}.
\newblock \emph{arXiv preprint arXiv:2509.25808}.

\bibitem[{Zhao and Tresp(2018)}]{ebp}
Rui Zhao and Volker Tresp. 2018.
\newblock \href {https://proceedings.mlr.press/v87/zhao18a.html} {Energy-based
  hindsight experience prioritization}.
\newblock In \emph{Proceedings of the 2nd Conference on Robot Learning
  ({CoRL})}, volume~87 of \emph{Proceedings of Machine Learning Research},
  pages 113--122. {PMLR}.

\bibitem[{Zheng et~al.(2025{\natexlab{a}})Zheng, Zhao, and Chen}]{m2po}
Haizhong Zheng, Jiawei Zhao, and Beidi Chen. 2025{\natexlab{a}}.
\newblock \href {https://arxiv.org/abs/2510.01161} {Prosperity before collapse:
  How far can off-policy {RL} reach with stale data on {LLMs}?}
\newblock \emph{arXiv preprint arXiv:2510.01161}.

\bibitem[{Zheng et~al.(2025{\natexlab{b}})Zheng, Zhou, Bartoldson, Kailkhura,
  Lai, Zhao, and Chen}]{greso}
Haizhong Zheng, Yang Zhou, Brian~R. Bartoldson, Bhavya Kailkhura, Fan Lai,
  Jiawei Zhao, and Beidi Chen. 2025{\natexlab{b}}.
\newblock \href
  {https://proceedings.neurips.cc/paper_files/paper/2025/hash/b42fb82fc3ffb88b872c4714f093875d-Abstract-Conference.html}
  {Act only when it pays: Efficient reinforcement learning for {LLM} reasoning
  via selective rollouts}.
\newblock In \emph{Advances in Neural Information Processing Systems},
  volume~38, pages 124321--124346. Curran Associates, Inc.

\bibitem[{Zhong et~al.(2025)Zhong, Zhang, Song, Hu, Jin, Wu, Chen, Chen, Zhou,
  Wan, Zhou, Jiang, Zhu, and Jiang}]{streamrl}
Yinmin Zhong, Zili Zhang, Xiaoniu Song, Hanpeng Hu, Chao Jin, Bingyang Wu, Nuo
  Chen, Yukun Chen, Yu~Zhou, Changyi Wan, Hongyu Zhou, Yimin Jiang, Yibo Zhu,
  and Daxin Jiang. 2025.
\newblock \href {https://arxiv.org/abs/2504.15930} {{StreamRL}: Scalable,
  heterogeneous, and elastic {RL} for {LLMs} with disaggregated stream
  generation}.
\newblock \emph{arXiv preprint arXiv:2504.15930}.

\bibitem[{Zhu et~al.(2026{\natexlab{a}})Zhu, Ren, Li, Lin, Yang, Liu, Zhen,
  Liu, and Zhang}]{dppo-prune}
Haodong Zhu, Yangyang Ren, Yanjing Li, Mingbao Lin, Linlin Yang, Xuhui Liu,
  Xiantong Zhen, Haiguang Liu, and Baochang Zhang. 2026{\natexlab{a}}.
\newblock \href {https://arxiv.org/abs/2603.04135} {Unbiased dynamic pruning
  for efficient group-based policy optimization}.
\newblock \emph{arXiv preprint arXiv:2603.04135}.

\bibitem[{Zhu et~al.(2026{\natexlab{b}})Zhu, He, Hou, Zhang, Zeng, Peng, Fang,
  and Yu}]{pspo}
Xinxin Zhu, Ying He, Haowen Hou, Ruichong Zhang, Nianbo Zeng, Yulin Peng,
  Jiongfeng Fang, and F.~Richard Yu. 2026{\natexlab{b}}.
\newblock \href {https://doi.org/10.1609/aaai.v40i34.40157} {{PSPO}:
  Prompt-level prioritization and experience-weighted smoothing for efficient
  policy optimization}.
\newblock In \emph{Proceedings of the {AAAI} Conference on Artificial
  Intelligence}, volume~40, pages 29186--29194. {AAAI} Press.

\end{thebibliography}

\clearpage
\appendix
\section{Full Experimental Details}
\label{app:training-setup}

\paragraph{Artifacts and Licenses}
We use only publicly released artifacts.
The base models are the Qwen3-Base family~\citep{qwen3} at 0.6B, 1.7B, and 4B parameters, released under the Apache 2.0 license.
The training prompts are from DeepScaleR-Preview~\citep{deepscaler}, released under the MIT license.
Evaluation uses MATH-500~\citep{lightman2023lets} along with AIME25~\citep{aime2025}, AIME26~\citep{aime2026}, HMMT-F25~\citep{hmmt_feb2025}, and HMMT-F26~\citep{hmmt_feb2026}.
The MATH-500 source problems are MIT-licensed via the upstream MATH benchmark and the PRM800K repository.
The AIME and HMMT problems are publicly released by the Mathematical Association of America and the Harvard-MIT Mathematics Tournament, respectively.
All prompts and responses are in English.
None of these resources contain personal information about individuals.
Our use is consistent with the research-use terms of the original releases.
We release no new dataset or pretrained model artifact with this paper.

\paragraph{Benchmarks and Decoding}
We evaluate on the five math-reasoning benchmarks listed in the Artifacts and Licenses paragraph above.
All five are evaluation-only.
The sizes are $500$ problems for MATH-500, $30$ for AIME25, $30$ for AIME26, $30$ for HMMT-F25, and $33$ for HMMT-F26.
AIME and HMMT are reported per year, comparing AIME25 against AIME26 and HMMT-F25 against HMMT-F26 to show per-benchmark variance.
For each prompt we draw samples at temperature $0.6$ and top-$p$ at $0.8$.
MATH-500 is evaluated at $k = 8$.
AIME and HMMT are evaluated at $k = 16$ as two disjoint waves of $8$.
Per benchmark we report AVG@$K$ for mean accuracy and PASS@$K$ for any-correct rate.
The ablation tables from Table~\ref{tab:age-clipping} to Table~\ref{tab:priority-signal} report a five-benchmark aggregate.
This aggregate is the unweighted mean of MATH-500, AIME25, AIME26, HMMT-F25, and HMMT-F26.
The main result in Table~\ref{tab:cross-scale-main} shows the per-benchmark cells directly.
The $k$-wave construction and the full seed policy are detailed in the Training and Sampling Seeds paragraph below.

\paragraph{Training Set and Hyperparameters}
The training set is DeepScaleR-Preview~\citep{deepscaler}, filtered to prompts of at most $1024$ tokens, leaving $40{,}306$ prompts.
Table~\ref{tab:training-setup-full} lists the RL training hyperparameters used across all conditions.
The per-scale row reports the value that varies by model size.
The tabulated results are all at global step~200, which is fixed across conditions to equalize the rollout budget.
GRPO uses a group size of $G = 8$ rollouts per prompt at every scale.
This is the per-prompt sampling count behind the $\{A_i\}$ computation in Section~\ref{sec:prelim} and the $G$-rollout group in Algorithm~\ref{alg:ours}.

\begin{table}[!t]
\centering\footnotesize
\setlength{\tabcolsep}{8pt}
\resizebox{\ifdim\width>\columnwidth \columnwidth\else \width\fi}{!}{%
\begin{tabular}{l l}
\toprule
Setting & Value \\
\midrule
\multicolumn{2}{l}{\emph{Common across scales}} \\
Train batch size                    & 256 \\
Group size $G$                      & 8 \\
Max prompt length                   & 1024 tokens \\
Max response length                 & 8192 tokens \\
Learning rate                       & $1 \times 10^{-6}$ \\
PPO mini-batch size                 & 128 \\
Clip $\epsilon_{\text{low}}$, $\epsilon_{\text{high}}$, $c_{\text{dual}}$ & 0.2, 0.28, 10.0 \\
Advantage stabilizer $\epsilon_\sigma$ (Eq.~\ref{eq:grpo-adv}) & $10^{-6}$ \\
Optimizer                           & AdamW ($\beta_1{=}0.9$, $\beta_2{=}0.999$) \\
Weight decay                        & 0.01 \\
Gradient clip                       & 1.0 (L2 norm) \\
LR warmup                           & none \\
KL penalty $\beta$                  & 0 (\texttt{kl\_coef=0}, \texttt{use\_kl\_loss=False}) \\
\midrule
\multicolumn{2}{l}{\emph{Per-scale at 0.6B, 1.7B, and 4B}} \\
PPO micro-batch per GPU             & 4, 4, 2 \\
\bottomrule
\end{tabular}%
}
\caption{RL training hyperparameters.
         The shared block applies to all conditions and the per-scale row varies across 0.6B, 1.7B, and 4B.
         $\epsilon_{\text{low}}$ and $\epsilon_{\text{high}}$ are the DAPO clip-higher window from Section~\ref{sec:prelim}.
         $c_{\text{dual}}$ is the verl dual-clip bound on negative-advantage ratios, and the KL term is fully disabled.}
\label{tab:training-setup-full}
\end{table}

\paragraph{Compute}
All training runs use NVIDIA H100 80GB SXM5 GPUs in a single-node configuration.
The 0.6B and 1.7B runs use four H100s per run, and the 4B runs use eight.
A single run to step~200 takes approximately $10$, $7$, and $14$ wall-clock hours at 0.6B, 1.7B, and 4B.
These correspond to roughly $40$, $28$, and $112$ H100-hours per run.
The 0.6B wall-clock exceeds the 1.7B wall-clock because the 0.6B rollouts are on average substantially longer than the 1.7B rollouts.
The mean rollout lengths at step~200 are approximately $2{,}086$ and $1{,}191$ tokens respectively, as reported in Table~\ref{tab:training-length-cross-scale}.
Rollout generation dominates per-step time at fixed GPU count.
The full sweep reported in the paper covers our method, the naive baselines, and the ablations across the three scales, for $29$ distinct configurations.
It consumes roughly $1{,}200$ H100-hours to step~200, or about $2{,}500$ H100-hours including extended-training runs and the roughly $30\%$ overhead from failed or restarted jobs.
All runs use bfloat16 mixed precision and gradient checkpointing, with a rollout-side GPU memory utilization of $0.6$.

\paragraph{Software}
Training uses the verl framework~\citep{verl} at v0.7.0.dev0, extended with custom modifications for the replay buffer described in this paper.
Rollout generation during GRPO updates uses vLLM~\citep{vllm} at v0.11.0 with bfloat16 precision, tensor-parallel size $1$, a maximum model length of $9{,}216$ tokens, and chunked prefill.
The training-side runtime is PyTorch v2.8.0 on CUDA 12.8 with FlashAttention v2.8.1, FlashInfer v0.3.1, Hugging Face transformers v4.57.6, and Ray v2.51.1 for distributed orchestration.
Evaluation uses vLLM v0.14.1 with Hugging Face transformers v5.0.0.
The evaluation maximum response length is $8{,}192$ tokens at 0.6B and 1.7B and $16{,}384$ tokens at 4B.
Sampling settings are described in the Benchmarks and Decoding paragraph above.
Each prompt is rendered as a two-message chat with system content ``You are a helpful assistant.'' and user content ``Problem~: \{question\}\textbackslash n\textbackslash nPlease reason step by step, and put your final answer within \texttt{\textbackslash boxed\{\}}.''.
The chat is applied via \texttt{tokenizer.apply\_chat\_template(\ldots, add\_generation\_prompt=True)}.
All scales use the tokenizer of the corresponding Qwen3-Base checkpoint.

\paragraph{Verifier}
Answers are extracted from the final \texttt{\textbackslash boxed\{\ldots\}} expression within the trailing $500$ characters of the completion.
This truncation bounds verifier latency without changing correctness when the final boxed answer is near the end of the response.
Training-time reward verification uses a Minerva-style code-based comparator adapted from lm-evaluation-harness~\citep{lmevalharness}.
The extracted answer is normalized through LaTeX cleanup, unit-word removal, fraction and square-root shorthand expansion, and digit-comma stripping.
The normalized answer is then compared by exact string match to the similarly normalized ground truth, yielding a reward of $\pm 1$.
Final evaluation grading uses math-verify~\citep{mathverify} at v0.9.0 with the default extractor and grader.
Both the prediction and the ground truth are parsed with \texttt{math\_verify.parse}, wrapping the operand in a LaTeX math environment if not already inside one.
The answer is accepted when \texttt{math\_verify.verify(gt, pred)} returns true.
If either side fails to parse, we fall back to raw string equality.

\paragraph{Training and Sampling Seeds}
The cross-scale main comparison of Table~\ref{tab:cross-scale-main} is run three times per condition, with training seeds $42$, $101$, and $202$.
Its per-benchmark cells are the mean over those three runs, and its avg and pass columns add the standard deviation across them.
The AES and length analyses of Section~\ref{sec:dyn-aes} and the matched-FLOPs curves of Section~\ref{sec:dyn-compute} use the same three runs.
The ablations use training seed $42$ alone.
Evaluation sampling uses one primary $k = 8$ wave at seed $42$ for every benchmark.
Seed $42$ is the legacy default of our evaluation harness.
This wave is the sole source of MATH-500's $k = 8$ result.
For AIME and HMMT we additionally draw a second $k = 8$ wave at seed $43$ and concatenate the two waves to obtain the reported $k = 16$ statistics.
The two seeds yield disjoint sample\_id ranges $0$ through $7$ and $8$ through $15$ that do not overlap.
The same two seeds are used for every condition at every scale.
The conditions are GRPO, the naive baselines, and our method.
Any seed-dependent fluctuation in the per-example samples is therefore held constant across the comparisons in Tables~\ref{tab:cross-scale-main} through~\ref{tab:aes-base}.
This is not a source of differential bias between methods.

\section{Effective Batch Size After Zero-Variance Filtering}
\label{app:effective-batch}

Table~\ref{tab:effective-batch} reports the per-step surviving fresh-rollout count after zero-variance filtering for the GRPO baseline at the three Qwen3-Base scales.
The count is averaged over the first $200$ training steps, which is the snapshot point used throughout the paper.
Each step issues the same fresh-rollout budget $B_{\text{fresh}} = 2048$, which is the train batch of $256$ prompts at $G = 8$ rollouts each and matches the $B_{\text{fresh}}$ symbol in Algorithm~\ref{alg:ours}.
The fraction surviving the zero-variance filter varies by scale.
The surviving rollouts are those whose group has mixed correct and incorrect rewards.
Just over three-fifths of the smallest model's batch falls in all-wrong groups, leaving a mixed-reward fraction of approximately $34\%$ at 0.6B.
The larger models hold $44\%$ to $49\%$ mixed-reward groups.
The dropped fraction shifts toward all-correct groups as model scale increases.

A capacity-only buffer-eviction rule would therefore induce a scale-dependent staleness profile.
Larger models fill the buffer faster and evict sooner, while smaller models do the opposite.
Age eviction with horizon $\tau_{\max}$, as described in Section~\ref{sec:method-buffer}, removes this dependence by bounding worst-case rollout age independently of buffer fullness.

\begin{table}[!t]
\centering\footnotesize
\setlength{\tabcolsep}{4pt}
\resizebox{\ifdim\width>\columnwidth \columnwidth\else \width\fi}{!}{%
\begin{tabular}{l c c c c}
\toprule
Scale & $B_{\text{fresh}}$ & $B_{\text{fresh}}'$ & All-pos (\%) & All-neg (\%) \\
\midrule
Qwen3-0.6B & 2048 &  692.0 &  4.75 & 61.48 \\
Qwen3-1.7B & 2048 & 901.6 &  8.57 & 47.42 \\
Qwen3-4B   & 2048 & 999.2 & 21.09 & 30.10 \\
\bottomrule
\end{tabular}%
}
\caption{Effective batch size after zero-variance filtering for the GRPO baseline at three Qwen3-Base scales, averaged over training steps $0$ to $200$.
$B_{\text{fresh}}$ is the pre-filter rollout count from Algorithm~\ref{alg:ours}.
$B_{\text{fresh}}'$ is the mean surviving count belonging to mixed-reward groups.
``All-pos'' and ``All-neg'' are the percentages of groups filtered out for being all-correct and all-wrong.}
\label{tab:effective-batch}
\end{table}

\section{Replay-Ratio and Priority-Strength Sweep}
\label{app:ratio-alpha-full}

Table~\ref{tab:ratio-alpha-sweep} sweeps the replay ratio $r$ and priority-strength exponent $\alpha$ at Qwen3-1.7B-Base.
Our setting at $r = 0.5$ and $\alpha = 0.5$ leads on both metrics.
It is essentially tied for best on AVG@$K$, within $0.05$~pp of the next-best $r=1.0, \alpha=1.0$, and has a larger margin of $+1.68$~pp on PASS@$K$ over the next-best.
We adopt the setting on the PASS@$K$ margin.
PASS@$K$ rewards solving a problem on at least one of $k$ samples, which is the axis our rollout-level $|A_i|$ priority targets by recycling the within-group minority rollouts that fresh batches see sparsely (Section~\ref{sec:method-priority}).

\begin{table}[!t]
\centering\footnotesize
\setlength{\tabcolsep}{6pt}
\resizebox{0.8\columnwidth}{!}{%
\begin{tabular}{l cc}
\toprule
Config & avg & pass \\
\midrule
GRPO (no replay) & 15.75 & 27.64 \\
\midrule
Ours at $r = 0.5$, $\alpha = 0.5$ & \textbf{16.58}    & \textbf{32.35} \\
$r = 0.5$, $\alpha = 1.0$                 & 15.96             & 27.39 \\
$r = 1.0$, $\alpha = 0.5$                 & 15.99             & 27.99 \\
$r = 1.0$, $\alpha = 1.0$                 & 16.53 & 30.67 \\
\bottomrule
\end{tabular}%
}
\caption{Replay-ratio and priority-strength sweep at Qwen3-1.7B-Base.
         avg and pass are five-benchmark aggregates under Table~\ref{tab:cross-scale-main}'s $k$ convention.
         \textbf{Bold} marks the best.}
\label{tab:ratio-alpha-sweep}
\end{table}

\section{AES References and the Length Decomposition}
\label{app:aes-references}

Table~\ref{tab:aes-references} lists the reference values used in the AES computation of Table~\ref{tab:aes-base}.
Table~\ref{tab:length-references} gives the two correctness-conditioned sides whose ratio Table~\ref{tab:length-quality} reports in Section~\ref{sec:dyn-aes}.
Both tables share the benchmark set, the $k$ convention, and the decoding settings of the post-training evaluation in Section~\ref{sec:exp-setup}.
For AES, the Base init values are the untrained Qwen3-Base checkpoints, and the GRPO values are the no-replay GRPO baseline at step~200.

\begin{table}[!t]
\centering\footnotesize
\setlength{\tabcolsep}{5pt}
\resizebox{\ifdim\width>\columnwidth \columnwidth\else \width\fi}{!}{%
\begin{tabular}{l cc cc cc}
\toprule
\multirow{2}{*}{Reference}
 & \multicolumn{2}{c}{Qwen3-0.6B}
 & \multicolumn{2}{c}{Qwen3-1.7B}
 & \multicolumn{2}{c}{Qwen3-4B} \\
\cmidrule(lr){2-3}\cmidrule(lr){4-5}\cmidrule(lr){6-7}
 & Acc (\%) & $L$ & Acc (\%) & $L$ & Acc (\%) & $L$ \\
\midrule
Base init & $9.52$ & $2162$ & $13.53$ & $2007$ & $17.36$ & $2452$ \\
GRPO      & $11.18$ & $3609$ & $15.86$ & $2090$ & $27.77$ & $7031$ \\
\bottomrule
\end{tabular}%
}
\caption{AES reference values reporting five-benchmark aggregate accuracy in percent and per-rollout response length in Qwen3 tokens.
         These values serve as $A_{\text{ref}}$ and $L_{\text{ref}}$ in Table~\ref{tab:aes-base}.}
\label{tab:aes-references}
\end{table}

\begin{table}[!t]
\centering\footnotesize
\setlength{\tabcolsep}{5pt}
\resizebox{\ifdim\width>\columnwidth \columnwidth\else \width\fi}{!}{%
\begin{tabular}{l cc cc}
\toprule
\multirow{2}{*}{Scale} & \multicolumn{2}{c}{GRPO} & \multicolumn{2}{c}{Ours} \\
\cmidrule(lr){2-3}\cmidrule(lr){4-5}
 & $\mu(\text{len}\mid\checkmark)$ & $\mu(\text{len}\mid\times)$ & $\mu(\text{len}\mid\checkmark)$ & $\mu(\text{len}\mid\times)$ \\
\midrule
Qwen3-0.6B & $526$ & $3820$ & $661$ & $3631$ \\
Qwen3-1.7B & $507$ & $2262$ & $555$ & $2009$ \\
Qwen3-4B & $1047$ & $8844$ & $1110$ & $7546$ \\
\bottomrule
\end{tabular}%
}
\caption{Mean response length in tokens conditioned on correctness, pooled across the five-benchmark set.
         These are the two sides whose ratio Table~\ref{tab:length-quality} reports.
         That table forms the ratio within each run before averaging, so dividing the entries here does not reproduce its columns exactly.}
\label{tab:length-references}
\end{table}

\section{Rollout-Level $|A_i|$ vs.\ Query-Level $\sigma_g$ in Binary RLVR}
\label{app:binary-advantage-sigma}

We compare the two priority signals that recur throughout the paper.
These are the rollout-level absolute advantage $|A_i|$ and the query-level standard deviation $\sigma_g$.
The comparison is in the regime of binary verifiable rewards $r_i \in \{-1, +1\}$ that our experiments use.
Both quantities admit closed forms parameterized by the group's correct count $k$.
We drop the stabilizer $\epsilon_\sigma$ of Equation~\ref{eq:grpo-adv}, which is smaller than $\sigma_g$ by six orders of magnitude for any mixed group.

\begin{proposition}[Closed-form $|A_i|$ and $\sigma_g$ in binary RLVR]
\label{prop:minority}
Consider a group of $G$ rollouts under a binary verifiable reward
$r_i \in \{-1, +1\}$ with correct count $k$ and $0 < k < G$. The
group mean and standard deviation are
\begin{align}
  \mu_g &\;=\; \tfrac{2k}{G} - 1,
    \label{eq:binary-mu} \\
  \sigma_g &\;=\; 2\sqrt{\tfrac{k}{G}\!\left(1 - \tfrac{k}{G}\right)},
    \label{eq:binary-sigma}
\end{align}
and the group-relative advantage
$A_i = (r_i - \mu_g)/\sigma_g$ takes exactly two values per group,
\begin{align}
  |A_i^{\,\text{correct}}|   &\;=\; \sqrt{\tfrac{G - k}{k}},
    \label{eq:abs-adv-correct} \\
  |A_i^{\,\text{incorrect}}| &\;=\; \sqrt{\tfrac{k}{G - k}}.
    \label{eq:abs-adv-incorrect}
\end{align}
For any mixed group with $0 < k < G$, the side with fewer rollouts carries strictly the larger $|A_i|$.
We call this side the \emph{minority}.
\end{proposition}

\begin{proof}
With $r_i \in \{-1, +1\}$, the group statistics are
\begin{align*}
  \mu_g &\;=\; \tfrac{k\cdot(+1) + (G - k)\cdot(-1)}{G} \;=\; \tfrac{2k}{G} - 1, \\
  \sigma_g^2 &\;=\; \mathbb{E}[r^2] - \mu_g^2
              \;=\; 1 - \left(\tfrac{2k}{G} - 1\right)^2 \\
             &\;=\; 4\,\tfrac{k}{G}\!\left(1 - \tfrac{k}{G}\right),
\end{align*}
yielding Equations~\ref{eq:binary-mu} and~\ref{eq:binary-sigma}.
Substituting into $A_i = (r_i - \mu_g)/\sigma_g$,
\begin{align}
  A_i^{\,\text{correct}}
    &\;=\; \tfrac{(+1) - (2k/G - 1)}{\sigma_g} \notag \\
    &\;=\; \tfrac{2(G - k)/G}{2\sqrt{(k/G)(1 - k/G)}} \notag \\
    &\;=\; \sqrt{\tfrac{G - k}{k}},
    \label{eq:adv-correct} \\
  A_i^{\,\text{incorrect}}
    &\;=\; \tfrac{(-1) - (2k/G - 1)}{\sigma_g} \notag \\
    &\;=\; \tfrac{-2k/G}{2\sqrt{(k/G)(1 - k/G)}} \notag \\
    &\;=\; -\sqrt{\tfrac{k}{G - k}},
    \label{eq:adv-incorrect}
\end{align}
yielding the absolute values stated.
Their ratio is
\begin{equation}
\label{eq:minority-ratio}
  \tfrac{|A_i^{\,\text{correct}}|}{|A_i^{\,\text{incorrect}}|}
  \;=\; \tfrac{G - k}{k},
\end{equation}
which exceeds $1$ iff $k < G - k$, that is, iff the correct rollouts are the minority.
The smaller of $\{k, G - k\}$ therefore always indexes the side with the larger $|A_i|$.
\end{proof}

\paragraph{Implication for Priority Design}
$\sigma_g$ in Equation~\ref{eq:binary-sigma} depends only on the group's correct count $k$.
It is therefore a \emph{single} value shared by all $G$ rollouts of the same group.
Using $\sigma_g$ as a replay priority therefore promotes or demotes an entire group uniformly, with no within-group discrimination.
In contrast, $|A_i|$ from Equations~\ref{eq:abs-adv-correct} and~\ref{eq:abs-adv-incorrect} varies \emph{within} a group.
The minority side always carries strictly the larger $|A_i|$ by the ratio in Equation~\ref{eq:minority-ratio}.
The minority rollouts of skewed groups are the rare-correct or rare-incorrect examples.
They therefore receive a structurally higher priority under rollout-level $|A_i|$ than under query-level $\sigma_g$ sampling.
This is the formal basis for the rollout-level $|A_i|$ priority adopted in Section~\ref{sec:method-priority}.

\paragraph{Beyond Binary Rewards}
Proposition~\ref{prop:minority} relies on $r_i \in \{-1, +1\}$, or equivalently any two-valued encoding.
The closed forms in Equations~\ref{eq:abs-adv-correct} and~\ref{eq:abs-adv-incorrect} and the strict minority-side dominance no longer hold under continuous or partial-credit rewards.
The broader \emph{intuition} is that rollout-level $|A_i|$ tracks the rollouts carrying the most informative gradient signal.
This intuition is expected to extend to non-binary settings.
The within-group ranking property exploited here would have to be restated.

\section{Cross-Scale Training-Time Response Length}
\label{app:training-length-cross-scale}

We report the training-time response length of our method and the no-replay GRPO baseline over the first $200$ training steps at the three Qwen3-Base scales.
Figure~\ref{fig:training-length-cross-scale} plots the wandb-logged mean response length over training.
Table~\ref{tab:training-length-cross-scale} gives the step-$200$ endpoint of each curve along with the relative change of our method against the per-scale GRPO baseline.

\begin{figure*}[!t]
  \centering
  \includegraphics[width=\textwidth]{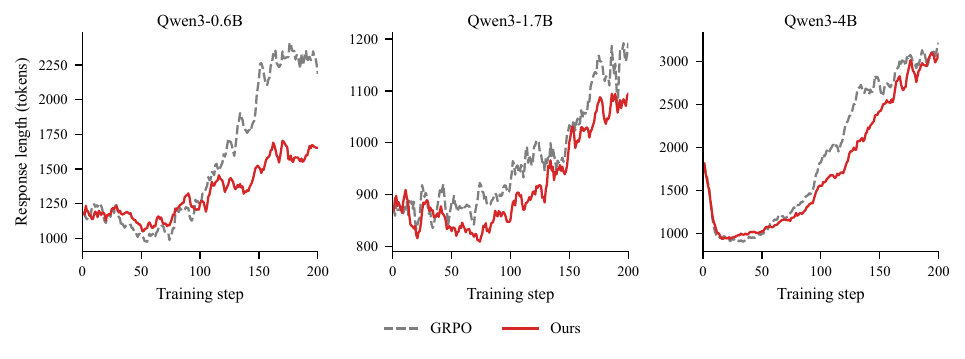}
  \caption{Training-time response length over the first $200$ steps for our method and the no-replay GRPO baseline at three Qwen3-Base scales.
           The curves are plotted as 5-step rolling means of \texttt{response\_length/mean}.
           The step-$200$ endpoint of each curve is the value tabulated in Table~\ref{tab:training-length-cross-scale}.}
  \label{fig:training-length-cross-scale}
\end{figure*}

\begin{table}[!t]
\centering\footnotesize
\setlength{\tabcolsep}{4pt}
\resizebox{\ifdim\width>\columnwidth \columnwidth\else \width\fi}{!}{%
\begin{tabular}{l c c c c c c}
\toprule
\multirow{2}{*}{Method}
 & \multicolumn{2}{c}{Qwen3-0.6B}
 & \multicolumn{2}{c}{Qwen3-1.7B}
 & \multicolumn{2}{c}{Qwen3-4B} \\
\cmidrule(lr){2-3}\cmidrule(lr){4-5}\cmidrule(lr){6-7}
 & $L$ & $\Delta L\%$ & $L$ & $\Delta L\%$ & $L$ & $\Delta L\%$ \\
\midrule
GRPO (no replay)                & 2086 & n/a     & 1191 & n/a    & 3443 & n/a    \\
Ours ($|A_i|$ rollout) & 1752 & $-16.0$ & 1099 & $-7.7$ & 3156 & $-8.3$ \\
\bottomrule
\end{tabular}%
}
\caption{Training-time response length at step~200 for our method against the no-replay GRPO baseline across three Qwen3-Base scales.
         $L$ is the wandb-reported \texttt{response\_length/mean} and is the endpoint of Figure~\ref{fig:training-length-cross-scale}.
         $\Delta L\%$ is the change of our method relative to the per-scale GRPO baseline.}
\label{tab:training-length-cross-scale}
\end{table}

\section{FLOPs Accounting}
\label{app:flops-accounting}

A forward-backward update costs about $6N$ FLOPs per trained token and a forward-only pass about $2N$ per token, where $N$ is the non-embedding parameter count~\citep{kaplan2020scaling}.
Both conditions run the same model at a given scale, so $N$ cancels and every FLOPs ratio in Section~\ref{sec:dyn-compute} reduces to a ratio of token counts taken from the training logs.
Matching learner FLOPs rather than total FLOPs is the conservative direction, because our method also generates fewer tokens than the baseline at equal learner FLOPs, so matching total FLOPs would move the comparison further in its favor.

\begin{figure*}[!t]
  \centering
  \includegraphics[width=\textwidth]{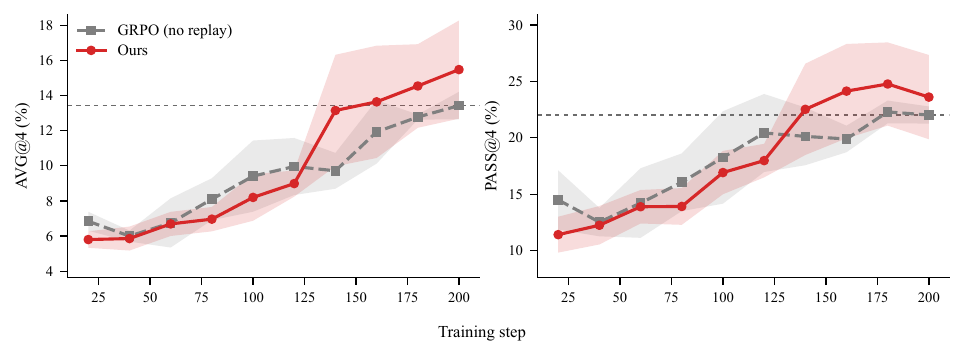}
  \caption{The points of Figure~\ref{fig:matched-flops} against training steps rather than learner FLOPs, at Qwen3-4B-Base.
           The dashed horizontal line marks the baseline's converged accuracy, which our method reaches at step $151$ on AVG@$K$ and step $138$ on PASS@$K$.}
  \label{fig:training-curve-steps}
\end{figure*}

Figure~\ref{fig:training-curve-steps} replots the same evaluations against training steps, which is the view a reader looking for a conventional learning curve expects.

\section{System Overhead}
\label{app:overhead}

\begin{table}[!t]
\centering\footnotesize
\setlength{\tabcolsep}{4pt}
\resizebox{\ifdim\width>\columnwidth \columnwidth\else \width\fi}{!}{%
\begin{tabular}{l cc cc cc}
\toprule
\multirow{2}{*}[-0.5ex]{Metric}
 & \multicolumn{2}{c}{Qwen3-0.6B}
 & \multicolumn{2}{c}{Qwen3-1.7B}
 & \multicolumn{2}{c}{Qwen3-4B} \\
\cmidrule(lr){2-3}\cmidrule(lr){4-5}\cmidrule(lr){6-7}
 & GRPO & Ours & GRPO & Ours & GRPO & Ours \\
\midrule
Step time (s)            & 191.0 & 165.8 & 106.6 & 114.7 & 193.9 & 241.4 \\
\quad update stage (s)   &  16.3 &  26.2 &  19.8 &  31.2 &  40.0 &  93.0 \\
Peak GPU memory (GB)     &  85.2 &  84.3 &  94.5 &  93.5 &  68.6 &  68.6 \\
\midrule
Buffer peak occupancy    & n/a & 7{,}840 & n/a & 9{,}568 & n/a & 12{,}344 \\
Buffer host RAM (GiB)    & n/a & 3.8 & n/a & 4.7 & n/a & 6.0 \\
\bottomrule
\end{tabular}%
}
\caption{Learner overhead of replay against the no-replay GRPO baseline, averaged over training steps $1$ to $200$ from the training logs.
         Hardware configurations differ across scales, so conditions should be compared within a scale rather than across scales.}
\label{tab:overhead}
\end{table}

Replay rollouts are appended as additional gradient-accumulation micro-steps rather than a larger per-device micro-batch, so Table~\ref{tab:overhead} shows peak GPU memory unchanged or slightly lower at every scale, with the buffer held in driver host memory at a fixed $0.5$~MiB per stored rollout.
The added cost falls on the update stage, and per-step wall time changes by $-13\%$, $+8\%$, and $+24\%$ at 0.6B, 1.7B, and 4B, since generation time falls as our method produces shorter responses.

\section{The 8B Result and the Start-of-Training Solve-Count Distribution}
\label{app:kprobe}

\begin{table*}[!t]
\centering\footnotesize
\setlength{\tabcolsep}{5pt}
\resizebox{\ifdim\width>\textwidth \textwidth\else \width\fi}{!}{%
\begin{tabular}{l cccc}
\toprule
 & Qwen3-0.6B & Qwen3-1.7B & Qwen3-4B & Qwen3-8B \\
\midrule
$k = 0$ & $74.24$ & $63.30$ & $49.54$ & $38.00$ \\
$k = 1$ & $12.14$ & $15.28$ & $18.96$ & $17.48$ \\
$k = 2$ & $5.74$ & $8.20$ & $12.70$ & $12.48$ \\
$k = 3$ & $3.30$ & $5.62$ & $8.98$ & $9.38$ \\
$k = 4$ & $2.24$ & $3.46$ & $6.10$ & $7.92$ \\
$k = 5$ & $1.54$ & $2.24$ & $2.64$ & $6.28$ \\
$k = 6$ & $0.68$ & $1.48$ & $0.90$ & $5.12$ \\
$k = 7$ & $0.10$ & $0.40$ & $0.18$ & $2.74$ \\
$k = 8$ & $0.02$ & $0.02$ & $0.00$ & $0.60$ \\
\midrule
$\mathbb{E}[k]$ & $0.55$ & $0.85$ & $1.16$ & $1.88$ \\
\midrule
Mixed, $1 \le k \le 7$ & $25.7$ & $36.7$ & $50.5$ & $61.4$ \\
\quad of which $k \le 3$ & $\mathbf{21.2}$ & $\mathbf{29.1}$ & $\mathbf{40.6}$ & $\mathbf{39.3}$ \\
\quad of which $k \ge 4$ & $\mathbf{4.6}$ & $\mathbf{7.6}$ & $\mathbf{9.8}$ & $\mathbf{22.1}$ \\
\bottomrule
\end{tabular}%
}
\caption{Solve-count distribution of the four base models before any training, as the percentage of $5{,}000$ training prompts whose $8$ rollouts contain $k$ correct answers, measured under the training sampling settings.
         Mixed groups are those the zero-variance filter keeps.
         Within them, $k \le 3$ are the groups whose minority side is correct, which is where the highest $|A_i|$ priorities sit, and $k \ge 4$ are the rest.
         Each bin has a standard error of at most $0.7$~pp.}
\label{tab:kprobe-initial-k}
\end{table*}

\begin{table*}[!t]
\centering\footnotesize
\setlength{\tabcolsep}{4pt}
\resizebox{\ifdim\width>\textwidth \textwidth\else \width\fi}{!}{%
\begin{tabular}{l ccccc cc}
\toprule
 & MATH-500 & AIME25 & AIME26 & HMMT-F25 & HMMT-F26 & avg & pass \\
\midrule
GRPO & $86.45$ & $17.50$ & $19.17$ & $12.08$ & $14.02$ & $29.84$ & $44.45$ \\
Ours & $86.62$ & $19.17$ & $19.17$ & $11.25$ & $13.64$ & $29.97$ & $43.83$ \\
\bottomrule
\end{tabular}%
}
\caption{Qwen3-8B-Base at global step~200 against its own no-replay GRPO baseline on the five-benchmark set.
         Per-benchmark cells and the avg column are AVG@$K$, and the pass column is PASS@$K$, both at $k = 8$ on every benchmark.
         The run uses a train batch of $128$ and a PPO mini-batch of $64$ rather than the $256$ and $128$ of the body results, and it uses a single training seed.
         Its pass column is therefore not comparable to Table~\ref{tab:cross-scale-main}, which uses $k = 16$ on the contest benchmarks.
         All other training and decoding settings match Appendix~\ref{app:training-setup}.}
\label{tab:8b-result}
\end{table*}

Table~\ref{tab:kprobe-initial-k} reports the solve-count distribution of the four base models before any training.
A group whose minority side holds $m$ of the $8$ rollouts gives each of those rollouts $|A_i| = \sqrt{(8-m)/m}$, so this distribution fixes the distribution of priorities the buffer draws from.
At 4B the correct side is the rare one in $99\%$ of the groups carrying the highest priority.
From 0.6B to 4B the mass leaving $k = 0$ lands at $k \le 3$, where that is the case, and the region grows from $21.2\%$ to $40.6\%$ of prompts.
From 4B to 8B it does not: the region is flat at $39.3\%$ while $k \ge 4$ more than doubles from $9.8\%$ to $22.1\%$.

Table~\ref{tab:8b-result} reports a Qwen3-8B-Base run against its own no-replay baseline, where the five-benchmark average differs by $+0.13$~pp against $+1.66$~pp at 4B.
That run uses a train batch of $128$ and a PPO mini-batch of $64$ rather than $256$ and $128$, is evaluated with $k = 8$ on all five benchmarks rather than $k = 16$ on the contest sets, and uses a single training seed.
We can only conjecture that the stronger 8B base model shifts the solve-count distribution away from the region $|A_i|$ targets and that this contributes to the absent margin.
One run under a different batch does not establish it.

\section{Importance-Ratio Clipping over Training}
\label{app:clipping-curves}

Figure~\ref{fig:clipping-curves} shows the per-step importance-ratio clipping fraction over the first $200$ training steps for the four age caps swept in Table~\ref{tab:age-clipping}.
The curves separate fresh-side and replay-side clipping.
The four age caps share a fresh-side profile that is essentially $\tau_{\max}$-invariant.
On the left panel, all four series fluctuate within roughly $0.06\%$.
Replay-side clipping grows monotonically with the age cap and accumulates over training.
On the right panel, $\tau_{\max}{=}60$ ends near $0.9\%$, roughly three times the $\tau_{\max}{=}1$ value at the same step.
This asymmetry is the dynamic counterpart to the policy-lag reading of Section~\ref{sec:exp-staleness}.
What grows with $\tau_{\max}$ is specifically the replay-side clipping rate, while the on-policy fresh side is unaffected.

\begin{figure*}[!t]
  \centering
  \includegraphics[width=\textwidth]{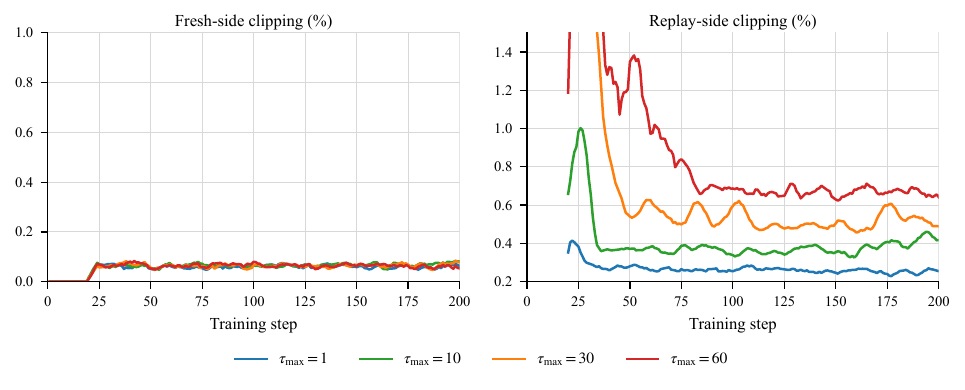}
  \caption{Importance-ratio clipping rates over training at Qwen3-1.7B-Base for the four age caps of Table~\ref{tab:age-clipping}.
           The curves are plotted as 5-step rolling means on the fresh side in the left panel and the replay side in the right panel.
           Replay-side clipping grows monotonically with the age cap while fresh-side clipping stays $\tau_{\max}$-invariant.
           This provides the dynamic counterpart to Table~\ref{tab:age-clipping}'s step-200 means.}
  \label{fig:clipping-curves}
\end{figure*}

\section{Full AES Ranking (12 Configurations)}
\label{app:aes-full}

Table~\ref{tab:aes-full} reports a 12-configuration 1.7B AES sweep against the no-replay GRPO baseline.
The sweep covers three axes all under fresh-anchored composition.
These are priority-strength $\alpha\in\{0, 0.5, 1.0\}$, max-age $\tau_{\max}\in\{1, 10, 30, 60\}$, and replay-ratio $r\in\{0.5, 1.0\}$.
The pool-composition naive baselines of Table~\ref{tab:aes-base} are reported there and are not duplicated here.
Our proposed $(|A_i|, \alpha{=}0.5, \tau_{\max}{=}10)$ leads the ranking at $+0.266$.
This is ahead of the next-best $(|A_i|, \alpha{=}0.5, \tau_{\max}{=}1)$ at $+0.248$ and every other tested configuration.
The tested configurations include alternative $(\alpha, \tau_{\max}, r)$ choices of the same $|A_i|$ priority and the uniform and $\sigma_g$ priority baselines.
The sweep uses a single training seed of $42$, so the 1.7B Ours row here at $+0.266$ is the seed-$42$ value of the corresponding cell in Table~\ref{tab:aes-base}, which averages three seeds.
Both use the same five-benchmark cache and aggregation reported in Appendix~\ref{app:aes-references}, but the reference here is the seed-$42$ GRPO run at $15.75$ accuracy and $2174$ tokens rather than the three-seed mean listed there.

\begin{table*}[!t]
\centering\footnotesize
\setlength{\tabcolsep}{4pt}
\resizebox{0.8\textwidth}{!}{%
\begin{tabular}{rl ccc cc}
\toprule
Rank & Method & Acc.\ (\%) & $L$ (tokens) & $\Delta L$ (\%) & $\Delta\mathrm{Acc}$ (\%) & AES \\
\midrule
 1 & Ours ($|A_i|$ rollout, default)        & 16.58 & 1938 & $+10.9$ & $+5.2$ & $\mathbf{+0.266}$ \\
 2 & $|A_i|$ rollout ($\tau_{\max}{=}1$)              & 16.49 & 1940 & $+10.8$ & $+4.7$ & $+0.248$ \\
 3 & $|A_i|$ rollout ($r{=}1.0$)                     & 15.99 & 1788 & $+17.8$ & $+1.5$ & $+0.224$ \\
 4 & $|A_i|$ rollout ($\tau_{\max}{=}30$)             & 16.01 & 1805 & $+17.0$ & $+1.6$ & $+0.219$ \\
 5 & uniform rollout (default)                       & 16.19 & 2121 & $+2.5$  & $+2.8$ & $+0.107$ \\
 6 & $|A_i|$ rollout ($\alpha{=}1.0$)                 & 15.96 & 2034 & $+6.5$  & $+1.3$ & $+0.105$ \\
 7 & uniform query (default)                         & 15.94 & 2065 & $+5.0$  & $+1.2$ & $+0.087$ \\
 8 & $\sigma_g$ rollout (default)                    & 16.04 & 2143 & $+1.4$  & $+1.8$ & $+0.069$ \\
 9 & $|A_i|$ rollout ($\tau_{\max}{=}60$)             & 15.28 & 1745 & $+19.7$ & $-3.0$ & $+0.047$ \\
10 & uniform rollout ($\tau_{\max}{=}30$)            & 15.18 & 1698 & $+21.9$ & $-3.6$ & $+0.038$ \\
11 & $\sigma_g$ query (default)                      & 15.43 & 2001 & $+8.0$  & $-2.0$ & $-0.022$ \\
12 & uniform rollout ($\tau_{\max}{=}1$)             & 16.45 & 2584 & $-18.8$ & $+4.4$ & $-0.056$ \\
\bottomrule
\end{tabular}%
}
\caption{Full token-based AES ranking of 12 1.7B configurations against the no-replay GRPO baseline at step~200 with Acc~$=15.75\%$ and $L=2174$~tokens.
         All rows are under fresh-anchored composition with AES formulation and 5-component aggregation following Table~\ref{tab:aes-base}.
         Default replay parameters are $(\alpha{=}0.5, \tau_{\max}{=}10, r{=}0.5)$ for $|A_i|$ and $\sigma_g$ rows and $(\alpha{=}0, \tau_{\max}{=}10, r{=}0.5)$ for uniform rows.
         Parenthetical qualifiers list only non-default values, and \textbf{bold} marks the best.}
\label{tab:aes-full}
\end{table*}

\section{Extended Related Work}
\label{app:extended-related-work}

Section~\ref{sec:related} positions our method against its closest neighbors along three lines.
The three lines are classical experience replay~\citep{dqn,per}, replay schemes for LLM RL~\citep{exgrpo,repo,rlep,buffer-matters,fatemi2026prioritized,arnal2026efficientrlreplay,freshper}, and priority signal design~\citep{cppo,dppo-prune,pods,dapo,vcrl,bae2026online,pcl}.
This appendix expands on the classical and priority strands, then summarizes where our method sits in the overall design space.

\subsection{Classical and Deep-RL Experience Replay}
\label{app:erw-classical-replay}

Experience replay originated with~\citet{lin1992self}, which stored past transitions and re-presented them to a connectionist Q-learner.
Deep replay developed with DQN~\citep{dqn} and was extended to continuous control by DDPG~\citep{ddpg}.
PER~\citep{per} introduced priority as an explicit component.
It samples proportional to $|\delta_i|^\alpha$ for TD error $\delta_i$ with importance-sampling corrections.
The $\alpha$ knob is the same one we ablate in Appendix~\ref{app:ratio-alpha-full}.
The max-priority-on-insertion convention of Section~\ref{sec:prelim} is inherited unchanged.
Rainbow~\citep{rainbow} identified PER as one of the higher-impact DQN components.
Distributional value learning~\citep{c51} is orthogonal in that it reshapes the target rather than the sampling distribution, much as GRPO already adopts a distributional view at the group level.

A second classical line addresses storage and eviction.
HER~\citep{her} relabels failed trajectories with achieved goals.
Energy-based hindsight prioritization~\citep{ebp} combines PER-style sampling with HER-style relabeling at the trajectory level, paralleling our move from transitions to rollouts.
LAP~\citep{lap} proves an equivalence between PER sampling and implicit loss reweighting and removes the need for importance-sampling correction.
The analogous correction is unnecessary here because $|A_i|$ in GRPO is itself the policy-gradient coefficient rather than a Bellman-target proxy.
ReF-ER~\citep{refer} drops gradient contributions whose ratio falls outside a trust region.
LFIW~\citep{lfiw} replaces importance-sampling ratios with a likelihood-free density-ratio estimator.
Both express the principle that a buffer implicitly defines a neighborhood of stale behavior policies.
Our age eviction with horizon $\tau_{\max}$ is a coarser combinatorial alternative that bounds the worst-case ratio under bounded per-step drift without estimating ratios or densities.
Selective replay~\citep{selective-replay} and topological replay~\citep{topological-replay} address what to store.
We instead place selection in the priority signal.

A third strand ties most directly to our composition choice.
\citet{cer} in ``Combined Experience Replay'' notes that buffer capacity is sensitive because small buffers correlate data and large ones make it stale.
That work always includes the most recent transition in every minibatch.
Fresh-anchored composition generalizes this from a single transition to a full GRPO batch.
CLEAR~\citep{clear} likewise mixes on-policy and replay updates each step with a behavioral-cloning regularizer.
We rely on the GRPO surrogate clip in its place.

Modern deep RL has also shown that the replay ratio is an important hyperparameter for sample efficiency.
REDQ~\citep{redq}, SR-SAC~\citep{srsac}, and the primacy-bias analysis of \citet{primacy-bias} stress the per-update axis.
\citet{fedus2020revisit} characterizes the joint effect of capacity and replay ratio in value-based deep RL.
LoRR~\citep{lorr} ports periodic-reset ideas to LLM RL.
In GRPO the per-rollout cost dominates the per-update cost by orders of magnitude, so the relevant ratio is gradient updates per generated rollout.
LLM policies drift more per step than small actor networks~\citep{areal,m2po}.
This is the failure mode the age cap is designed to prevent.
For stability, the deadly-triad analysis of \citet{deadly-triad} identifies replay as a load on the bootstrapping side.
GRPO does not bootstrap.
The off-policy axis is still present whenever stored rollouts differ from the current policy, which is the same concern $\tau_{\max}$ bounds.
Distributed replay systems such as Ape-X~\citep{apex}, R2D2~\citep{r2d2}, and Reverb~\citep{reverb} treat actor-learner lag as a scaling concern.
Their core abstractions of insertion-time priority, FIFO with optional eviction, and a configurable insert-to-sample ratio map onto our buffer directly.

\subsection{Priority Signals and Sample Selection in GRPO}
\label{app:erw-priority}

Priority signal design for GRPO inherits two open questions from the broader literature.
The first is which scalar identifies high-signal data.
The second is at what granularity that scalar is defined.
The first question appears in the TD-error of PER~\citep{per}, the loss equivalence of LAP~\citep{lap}, the density-ratio of LFIW~\citep{lfiw}, and the gradient-norm analysis of \citet{razin2024} that links low reward variance to vanishing gradients.
In GRPO, $|A_i|$ is the policy-gradient coefficient magnitude itself.
CPPO~\citep{cppo} and DPPO~\citep{dppo-prune} prune low-$|A|$ rollouts on the fresh side.
PODS~\citep{pods} keeps within-step reward extremes via $\sigma_g$-based downsampling.
We extend this $|A_i|$ lineage to the replay side by recycling individual rollouts with large absolute advantage after the single-pass update has discarded them.

The second question is granularity.
It runs from the prompt level~\citep{prompt-replay,pcl,vcrl,adarft,bae2026online,greso,no-prompt-left-behind,speedrl,dump,raftrej,reinforce-ada} to the query level used by most GRPO replay schemes including ExGRPO and~\citet{fatemi2026prioritized}.
Our method operates at the within-group rollout level.
Within-group difficulty signals such as $\sigma_g$~\citep{vcrl,fatemi2026prioritized} are reasonable query-level proxies but do not resolve the within-group minority rollouts where the gradient magnitude concentrates.
We argue this point in Section~\ref{sec:method-priority} and report results in Section~\ref{sec:exp-priority-signal}.
The prompt-level branch composes orthogonally with our buffer.
A prompt selector chooses what to roll out at all, while ours chooses what to reuse after generation.
CPPO, DPPO, and PODS are the fresh-side analogues that prune what enters the update, while we prioritize what re-enters it.

\subsection{Positioning Summary}
\label{app:erw-positioning}

Across these strands, our method occupies a specific point in a multidimensional design space.

\begin{itemize}[nosep,leftmargin=*]
  \item \textbf{Storage unit.} Rollout in ours, against
        transition in \citet{per}, trajectory in
        \citet{her}, and query or group in
        \citet{exgrpo,repo,rlep,buffer-matters}.
  \item \textbf{Priority signal.} $|A_i|$, the policy-gradient
        coefficient magnitude itself, against TD error in
        \citet{per}, correctness in
        \citet{raft,rlep,restem,raftrej}, correctness times
        entropy in \citet{exgrpo}, difficulty or $\sigma_g$
        in \citet{vcrl,fatemi2026prioritized}, and
        gradient-variance in \citet{gvmraft}.
  \item \textbf{Composition.} Fresh-anchored concat, against
        pooled in \citet{dqn,per}, mixed-batch in
        \citet{rlep,exgrpo,repo,buffer-matters,arnal2026efficientrlreplay,luffy,rlplus},
        and in-context replay in \citet{cer-language}.
  \item \textbf{Staleness control.} Age eviction
        $\tau_{\max}$, against capacity-only in
        \citet{dqn,per}, statistical clipping or
        importance-sampling bounds in
        \citet{topr,minimaxm1,m2po,bapo-clipping,fspo},
        system-level queue depth in
        \citet{areal,asyncflow,openrlhf,streamrl,llamarl,async-rlhf},
        soft recency bias in \citet{tba}, and reuse-count
        limit in \citet{ar3po}.
\end{itemize}

To our knowledge, no prior work occupies this combination of rollout-level $|A_i|$ priority with fresh-anchored composition under age eviction.
Per-axis ablations in Section~\ref{sec:exp-design} indicate that each choice contributes.
Removing the age cap increases replay-side clipping as shown in Table~\ref{tab:age-clipping}.
Removing fresh anchoring drops below GRPO at 4B as shown in Table~\ref{tab:pool-vs-concat-cross-scale}.
Reducing priority to query-level $\sigma_g$ removes the within-group variation of the signal and costs the most on PASS@$K$ as shown in Table~\ref{tab:priority-signal}.
The LLM-RL replay-buffer design space is tightly coupled, and the combination of choices matters more than any single axis.

\section{Use of AI Assistants}
\label{app:ai-assistance}

We used Claude Code from Anthropic as the primary AI assistant.
It supported routine code edits to training, evaluation, and analysis scripts.
It also provided grammar, phrasing, and editing assistance on the manuscript.
No AI tool was used to generate research ideas, experimental design, hypotheses, or scientific claims.
All results, analyses, and final wording were produced and verified by the authors.
This use is consistent with the ACL policy on AI writing assistance.

\end{document}